\documentclass[12pt]{article}
\usepackage{fullpage}
\usepackage[utf8]{inputenc}
\usepackage{hyperref}
\usepackage{amsthm,bm,amsfonts,amsmath,url}
\usepackage{CJKutf8}
\usepackage{caption}

\usepackage[backend=biber,style=nature,url=false]{biblatex}
\DeclareBibliographyCategory{main}
\DeclareBibliographyCategory{supplement}

\newcommand{\citemain}[1]{\mbox{\cite{#1}}\addtocategory{main}{#1}}
\newcommand{\citepmain}[1]{\mbox{\parencite{#1}}\addtocategory{main}{#1}}
\newcommand{\citesupp}[1]{\mbox{\cite{#1}}\addtocategory{supplement}{#1}}
\newcommand{\citepsupp}[1]{\mbox{\parencite{#1}}\addtocategory{supplement}{#1}}
\newcommand{\citecsupp}[2]{\mbox{\cite[#1]{#2}}\addtocategory{supplement}{#2}}

\newtheorem{thm}{Theorem}
\newtheorem{Assumption}{Assumption}
\newtheorem{Lemma}{Lemma}

\newtheorem{Definition}{Definition}

\usepackage[dvipsnames]{xcolor}
\usepackage{graphicx}
\usepackage{floatrow}
\usepackage{wrapfig}
\usepackage[normalem]{ulem} 
\usepackage{booktabs}
\usepackage{makecell}
\usepackage{array}
\usepackage{environ}
\newcommand{\centered}[2]{\parbox{#1\linewidth}{\centering #2}}

\usepackage{csquotes}
\renewcommand{\mkbegdispquote}[2]{\itshape}
\usepackage{enumerate}

\usepackage[most]{tcolorbox}

\newtcolorbox{pvaluefunctionbox}{
  title={Method for computing $p$-value functions},
  fonttitle=\bfseries,
  colback=white,
  colframe=gray,
  boxrule=0.8pt,
  arc=2pt,
  left=6pt,right=6pt,top=6pt,bottom=6pt,
  breakable
}

\usepackage{algorithm}
\usepackage{algpseudocode}

\newcommand{\bmu}{\boldsymbol{\bmu}}
\let\hat\widehat

\definecolor{vermilion}{rgb}{0.89, 0.26, 0.2}

\newcommand{\codecomment}[1]{\textbf{\color{black}// #1}}

\def\I{{\mathbb I}}
\def\P{{\mathbb P}}
\def\pr{\mathbb{P}}

\def\X{{\mathbf{X}}}
\def\x{{\mathbf{x}}}
\def\z{{\mathbf{z}}}

\def\z{{\mathbf{z}}} 
\def\E{{\mathbb E}}

\def\T{{\mathcal T}}

\def\B{{\mathbf{B}}}
\def\C{{\mathbf{C}}}

\let\hat\widehat

\NewEnviron{hidden}{%
  \setbox0\vbox{%
   \let\xxwrite\write
   \protected\def\write{\immediate\xxwrite}%
   {\tiny XX\BODY XX}}
  }

\renewcommand{\thefootnote}{\arabic{footnote}}

\title{Trustworthy scientific inference with generative models}
\author{James Carzon$^{1,\dagger}$ \and
        Luca Masserano$^{1,\dagger}$ \and
        Joshua D. Ingram$^{1,\dagger}$ \and
        Alex Shen$^{1,\dagger}$ \and
        Antonio Carlos Herling Ribeiro Junior$^{1}$ \and
        Tommaso Dorigo$^{2,3,4}$ \and
        Michele Doro$^{5,3}$ \and
        Joshua S. Speagle \begin{CJK*}{UTF8}{gbsn}(沈佳士)\end{CJK*}$^{6,7,8,9}$ \and
        Rafael Izbicki$^{10}$ \and
        Ann B. Lee$^{1,*}$}

\begin{document}

{\fontsize{6}{8}\selectfont
\footnotetext[1]{Department of Statistics and Data Science, Carnegie Mellon University, USA}
\footnotetext[2]{Lule\r{a} Tekniska Universitet, Sweden}
\footnotetext[3]{Istituto Nazionale di Fisica Nucleare (INFN), Sezione di Padova, Italy}
\footnotetext[4]{Universal Scientific Education and Research Network (USERN), Italy}
\footnotetext[5]{Department of Physics and Astronomy, Università di Padova, Italy}
\footnotetext[6]{Department of Statistical Sciences, University of Toronto, Canada}
\footnotetext[7]{David A. Dunlap Department of Astronomy \& Astrophysics, University of Toronto, Canada}
\footnotetext[8]{Dunlap Institute for Astronomy \& Astrophysics, University of Toronto, Canada}
\footnotetext[9]{Data Sciences Institute, University of Toronto, Canada}
\footnotetext[10]{Department of Statistics, Universidade Federal de São Carlos (UFSCar), Brazil}

\renewcommand{\thefootnote}{\fnsymbol{footnote}}
\def\thefootnote{$\dagger$}\footnotetext{These authors contributed equally to this work}
\def\thefootnote{*}\footnotetext{{\em Corresponding author:} Ann B. Lee,  annlee@andrew.cmu.edu}
}

\date{}
\maketitle

\begin{abstract}
Generative artificial intelligence (AI) excels at producing complex data structures (text, images, videos) by learning patterns from training examples. Across scientific disciplines, researchers are now applying generative models to ``inverse problems'' to directly predict hidden parameters from observed data along with measures of uncertainty. While these predictive or posterior-based methods can handle intractable likelihoods and large-scale studies, they can also produce biased or overconfident conclusions even without model misspecifications. We present a solution with Frequentist-Bayes (FreB), a mathematically rigorous protocol that reshapes AI-generated posterior probability distributions into (locally valid) confidence regions that consistently include true parameters with the expected probability, while  achieving minimum size when training and target data align. We demonstrate FreB's effectiveness by tackling diverse case studies in the physical sciences: identifying unknown sources under dataset shift, reconciling competing theoretical models, and mitigating selection bias and systematics in observational studies. By providing validity guarantees with interpretable diagnostics, FreB enables trustworthy scientific inference across fields where direct likelihood evaluation remains impossible or prohibitively expensive.
\end{abstract}

\newpage

\section{Introduction} 

How can scientists reliably infer internal properties of complex systems? This challenge---extracting underlying insights from the data we can collect---stands at the frontier of modern science across disciplines. Unfortunately, traditionally relied-upon statistical approaches \citepmain{casella2024statistical,vaart_asymptotic_1998} can break down precisely when dealing with the sophisticated (often computationally intractable) physics-based models needed to explore the most pressing questions~\citepmain{algeri_searching_2020,cranmer_frontier_2020}.
At the same time,  computational schemes like Markov Chain Monte Carlo (MCMC; \citepmain{brooks_handbook_2011}), nested sampling \citepmain{skilling_nested_2006} and GP emulators  \citepmain{gramacy_surrogates_2020} are not fast or rich enough to handle high-resolution large-scale data presented by, e.g., near-future astronomical surveys with billions of galaxies \citepmain{wang_sbi_2023} or complex spatio-temporal probability distributions \citepmain{saad_scalable_2024}.
\\

\noindent In response, many researchers are turning to a powerful alternative with generative artificial intelligence (AI) and neural inference.
Generative models and neural density estimators such as normalizing flows, diffusion models, and flow matching are used to generate plausible parameters for observed data by training on labeled examples. (Here we use the terms  \textit{label} and \textit{parameter} interchangeably to indicate internal properties of an object, e.g., the age of a galaxy or the mass of a subatomic particle.) By learning underlying patterns and structures of the train data, the result is a ``probability map'' connecting observations to plausible parameters---this map is known as a neural posterior distribution \citepmain{papamakarios_fast_2016,lueckmann_flexible_2017,radev_bayesflow_2022}. Such machine learning approaches bypass the need for computationally tractable mathematical formulas (likelihoods) while delivering results orders of magnitude faster than conventional approaches. Notable examples include applications with James Webb Space Telescope data \citepmain{wang_sbi_2023} and ocean remote sensing measurements \citepmain{sainsbury-dale_likelihood-free_2024,sainsbury-dale_neural_2025}. 
\\
\\
\noindent 
However, as we shall see, generative models (as well as other predictive and posterior-based approaches) can lead to misleading inferences if applied naively to parameter reconstruction (see also \citepmain{hermans_crisis_2022}). That is, despite recent promising advances, a fundamental question remains:

\begin{displayquote}
  Generative AI excels at \textbf{producing} complex data (text, images, videos), but how can scientists make sure that generative AI is equally successful at \textbf{recovering} hidden parameters from observed data with valid measures of uncertainty?
\end{displayquote}

\noindent The question, of reliably obtaining so-called valid confidence regions for parameters of interest, has profound implications across the sciences. In high-energy physics, CERN's Large Hadron Collider experiments analyze complex proton collision outcomes to measure Standard Model parameters \citepmain{glashow_renormalizability_1959,salam_weak_1959,weinberg_model_1967} and explore theoretical extensions like supersymmetry \citepmain{noauthor_quest_2025}. In astronomy, space telescopes like \emph{Gaia} determine stellar properties from spectral measurements \citepmain{gaia_collaboration_gaia_2023}. Similar inference challenges emerge in  geophysics \citepmain{stockman_sb-etas_2024}, epidemiology \citepmain{radev_outbreakflow_2021}, neuroscience \citepmain{goncalves_training_2020},  material science \citepmain{zhdanov_amortized_2022}, and molecular dynamics \citepmain{doi:10.1073/pnas.2420158122} to name a few.\\
\noindent 

\begin{figure}[t!]
    \centering
    \includegraphics[width=0.7\linewidth,trim={2cm 2cm 2cm 1cm},clip]{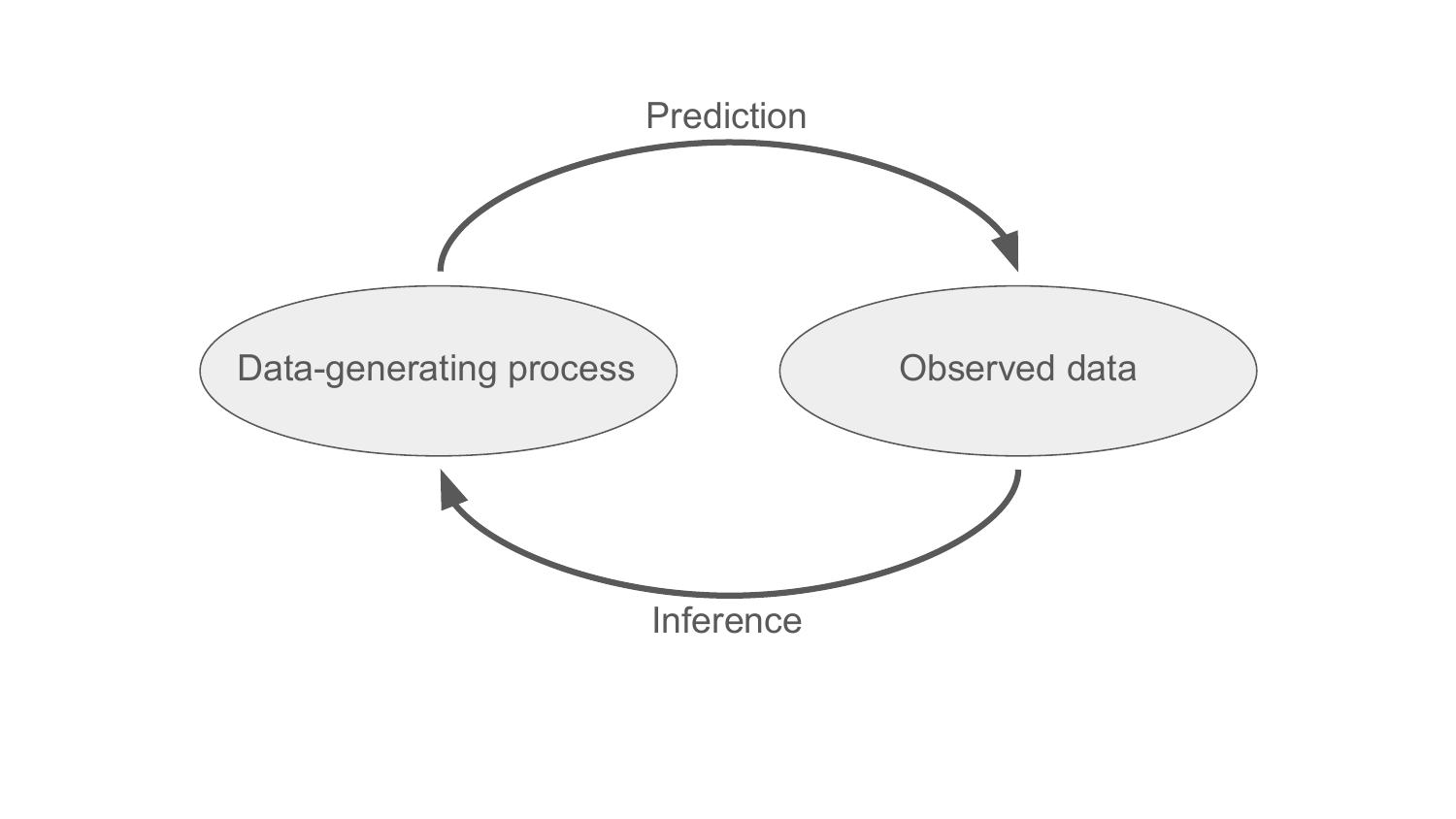}
    \caption{ {\bf Prediction versus inference.} Given a data-generating process---in the form of a scientific theory, statistical/AI model or Nature itself---scientists can make predictions on observable data (this is the so-called forward problem). In complex settings, the ``inverse'' problem of making trustworthy {\em inferences} on the underlying physical process (thereby allowing scientists to test and constrain scientific theories and the parameters defining them) can be challenging with limited data, even with a well-constructed forward model.}
    \label{fig:inference_and_prediction}
\end{figure}

\noindent The challenge of generative AI for inference lies in the fundamental difference between the {\em forward} problem of prediction (generating observations from known parameters $\theta$) and the {\em inverse} problem of inference (reconstructing parameters from observations $X$); see Figure~\ref{fig:inference_and_prediction}. For scientific discovery, parameter constraints must be statistically valid---scientists need confidence regions that contain the true parameter value with a specified probability or confidence level no matter what its true value is (local coverage\footnote{That is, we want a confidence set $C(X)$ such that $\P_{X|\theta}(\theta \in C(X)) \geq 1 - \alpha$,  for \emph{every} value of the unknown parameter $\theta$ at a prechosen miscoverage level $\alpha \in (0.1)$.}), while being sufficiently small to advance scientific understanding (high constraining power). Even with perfectly estimated probabilities under ideal conditions (such as an all-knowing simulator, or train data with no errors in labels), predictive or posterior-based approaches to learning parameters from data fail in two critical areas:
\begin{enumerate}
    \item \textbf{Lack of validity for individual instances.} Current methods may achieve correct coverage (or confidence level) of the true parameter when {\em averaged} across many objects/subjects, each with different parameter values. However, they provide no guarantees for each individual instance or state (parameter setting) and often lack the means to check local coverage for every possible parameter setting to pinpoint areas of potential failures. When scientists study individual stars, a specific particle collision event, or the current state of the atmosphere, this limitation leads directly to misleading conclusions as corresponding estimates may be overconfident, with deceivingly small regions of plausible parameters that have little chance of containing the true value. 
    \item \textbf{Biased results when train and target data do not match.} In practice, the examples used to train machine learning models rarely have the same properties as the objects of interest. Selection effects and observational limitations (when using real data) and competing theoretical models (when using simulated/synthetic data) all create mismatches between training examples and real-world targets. This problem fundamentally undermines reliability, especially because scientists do not know (and should not assume they can reliably guess) the true parameters of the targets in advance. The consequence is parameter estimates that are often unintentionally biased toward the values used to generate the train data---even when the truth is very different.
\end{enumerate}

\noindent To overcome these limitations, we introduce {\em Frequentist-Bayes} (FreB, pronounced as ``freebie'') {\em confidence procedures}---a mathematically rigorous and scalable protocol that reshapes probability distributions, such as those returned by neural density estimators and generative models, into statistically trustworthy parameter constraints. FreB uses a set of labeled examples to learn a transformation of posterior probability distributions to p-value functions via machine learning methods (as illustrated by the left column of Figure~\ref{fig:1D_example} in Section~\ref{sec:methods}). Level sets of the p-value functions then become confidence regions, maintaining proper local coverage by containing the true parameter with the stated probability. As long as some calibration data are available that come from the same physical process (likelihood) as the targets, FreB can account for misspecified models as well as differences in train and target data. Importantly, once calibrated, these procedures require no additional training when deployed on new data (they are ``amortized''), enabling efficient analysis of massive data sets.
\\
\\
The FreB framework offers three key advantages for scientific discovery:
\begin{enumerate}
\item \textbf{Works with small samples.} It provides reliable results even with just one observation per object---a common constraint in many scientific fields. There is also no need to simulate a batch of Monte Carlo samples per object, which is a computational bottleneck with traditional inference methods \citepmain{cousins_lectures_2024}.
\item \textbf{Guarantees for individual instances.} It ensures (and provides local diagnostics to verify) that stated confidence levels actually hold for each specific instance (that is, no matter the specific value of $\theta$ characterizing the property of, e.g., a star,  subatomic particle, human subject) being studied, not just on average across a population.
\item  \textbf{Precise when the prior and the forward model are accurate.} It produces tight, informative parameter constraints when scientists’ background knowledge (expressed as what is known as a prior distribution $\pi(\theta)$) and forward model (expressed by an approximate likelihood or simulator $\widehat{p}(X|\theta)$) align with the target data. We also arrive at tight parameter constraints for the target population with a high-fidelity model when the parameter distribution of the train data (here also just referred to as a ``prior'') is aligned with the true parameters of the targets.\footnote{\label{fn:pstr_prior} The prior and posterior distributions (mathematically denoted by $\pi(\theta)$ and $\pi(\theta|X)$, respectively) are typically interpreted as the uncertainty in our knowledge of $\theta$ {\em a priori} or {\em a posteriori} (before, and after the fact) of observing data $X$. In this work, we will use the terms ``priors'' and ``posteriors'' beyond the traditional subjective Bayesian view \citepmain{gelman_bayesian_2013} to also apply to probabilities that can be indirectly determined by the observed population of physical entities, such as the stars in our galaxy or different states of our climate system.}
\end{enumerate}

\noindent
In Section~\ref{sec:methods}, we give an overview of the FreB experimental set-up and protocol. In Section~\ref{sec:results}, we demonstrate FreB's effectiveness through a two-dimensional (2D) synthetic example and three diverse case studies in physics and astronomy, each case study addressing a specific statistical challenge (see Table~\ref{tab:case_studies}):

\begin{table}[p]
    \vspace*{-0.5in}
    \textbf{Scientific inference challenges addressed in our work}\vspace{-8mm}
    \caption{\small \textbf{Scientific inference challenges addressed in our work.} Each case study in Sections~\ref{sec:gamma_showers},~\ref{sec:Brutus_study} and~\ref{sec:Gaia_study} (with the set-up listed in the right column) illustrates a unique scientific challenge, which we resolve with our proposed approach. {\em Right Column,} {\bf I:} Ground-based detector array for measuring atmospheric cosmic-ray showers (Image credit: Richard White, MPIK). {\bf II:} Two differing models of the galaxy, simulated using \texttt{Brutus} \citepmain{speagle_deriving_2025} {\bf III:} Galactic map and noisy stellar labels (parameter estimates) included in a cross-match (\textcolor{orange}{orange}) between Gaia Data Release 3 \citepmain{gaia_collaboration_gaia_2023} and APOGEE Data Release 17 \citepmain{majewski_apache_2017}; a subsample with more precise labels are highlighted (\textcolor{blue}{blue}).}

    \begin{center}
        \resizebox{1\columnwidth}{!}{
            \begin{tabular}{
                >{\centering\arraybackslash}p{0.04\linewidth}|>{\centering\arraybackslash}p{0.42\linewidth}|>{\centering\arraybackslash}p{0.54\linewidth}}

                \toprule
                \textbf{\#}
                &
                \centered{1}{\textbf{Inference challenge}}
                & 
                {\textbf{Case study}} \\
            
                \midrule I 
                & 
                \textsc{
                    \centered{0.97}{Identify previously unknown \\ physical sources}
                }
                & 
                \makecell{
                    \includegraphics[width=5.5cm]{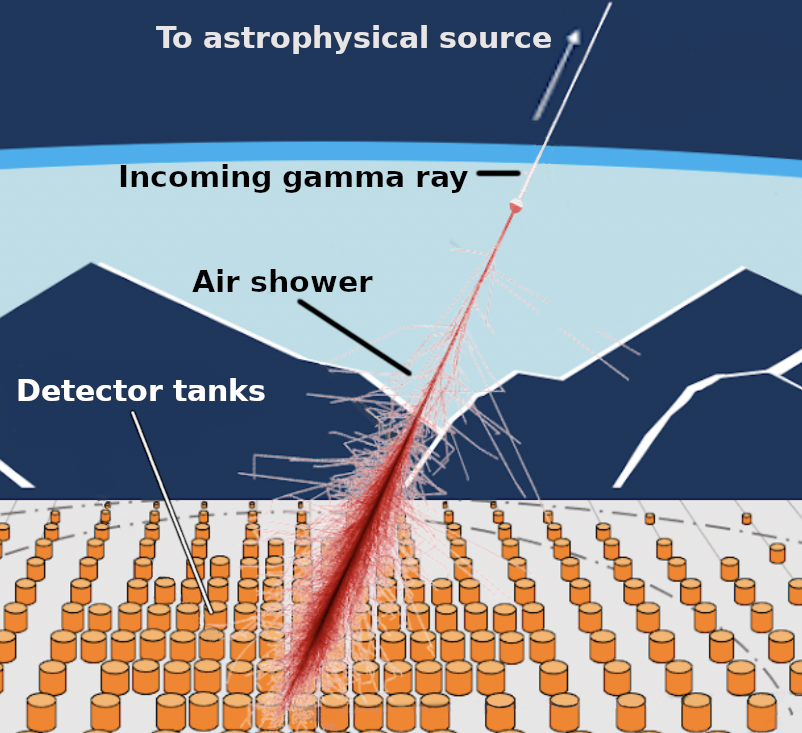}\\
                    \hline \small
                        Reconstructing gamma rays to localize \\ and measure astrophysical sources
                } \\
            
                \midrule II 
                & 
                \textsc{
                    \centered{0.97}{Resolve conflicting results from differing models of nature}
                }
                & 
                \makecell{
                    \includegraphics[width=5.5cm]{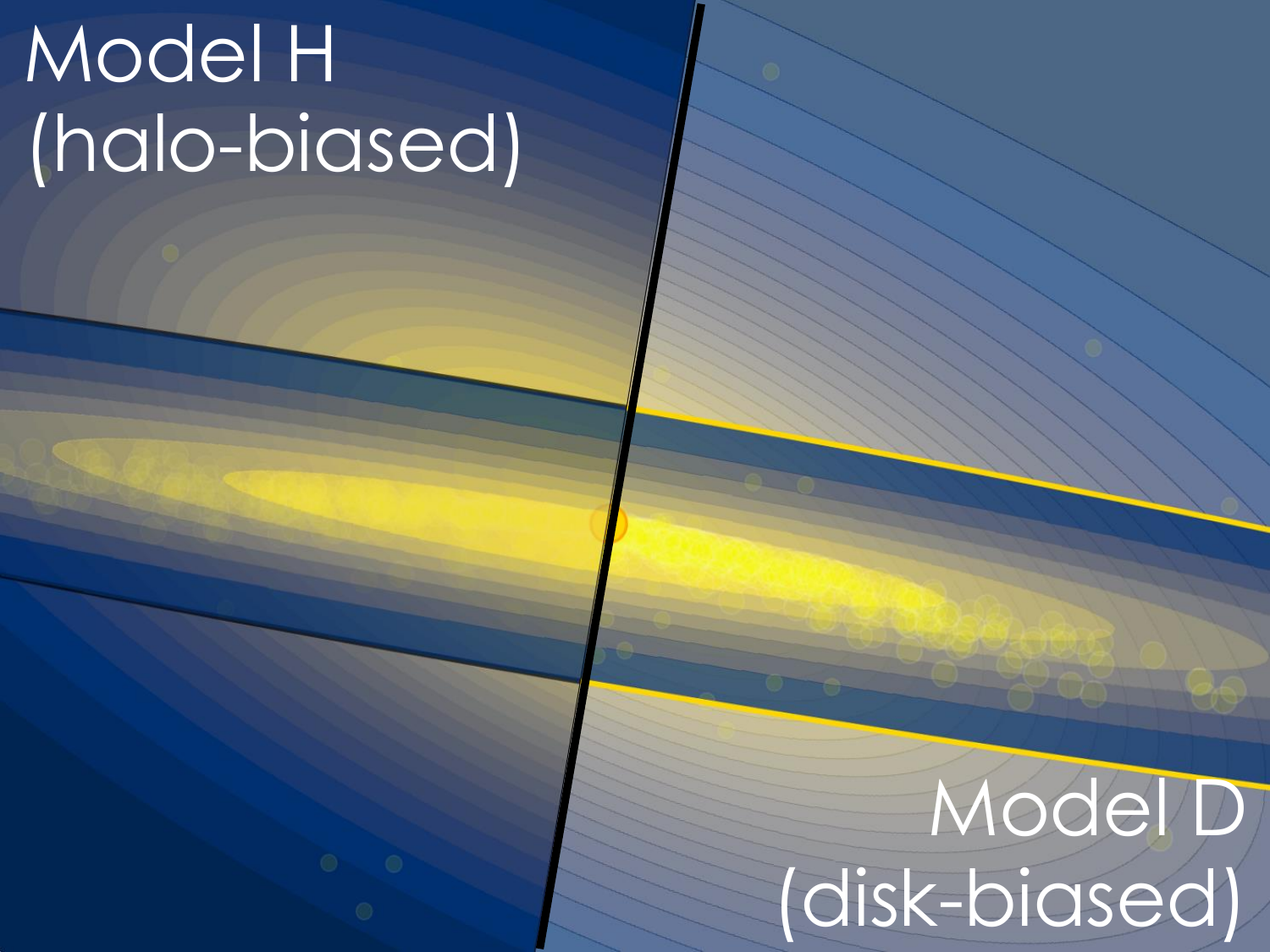}\\
                    \hline \small
                        Inferring properties of Milky Way stars \\ using two different galactic models
                } \\
            
                \midrule III
                &
                \textsc{
                    \centered{0.97}{Ensure trustworthy inferences \\ under selection bias and systematics}
                }
                &
                \makecell{
                    \includegraphics[width=5.5cm]{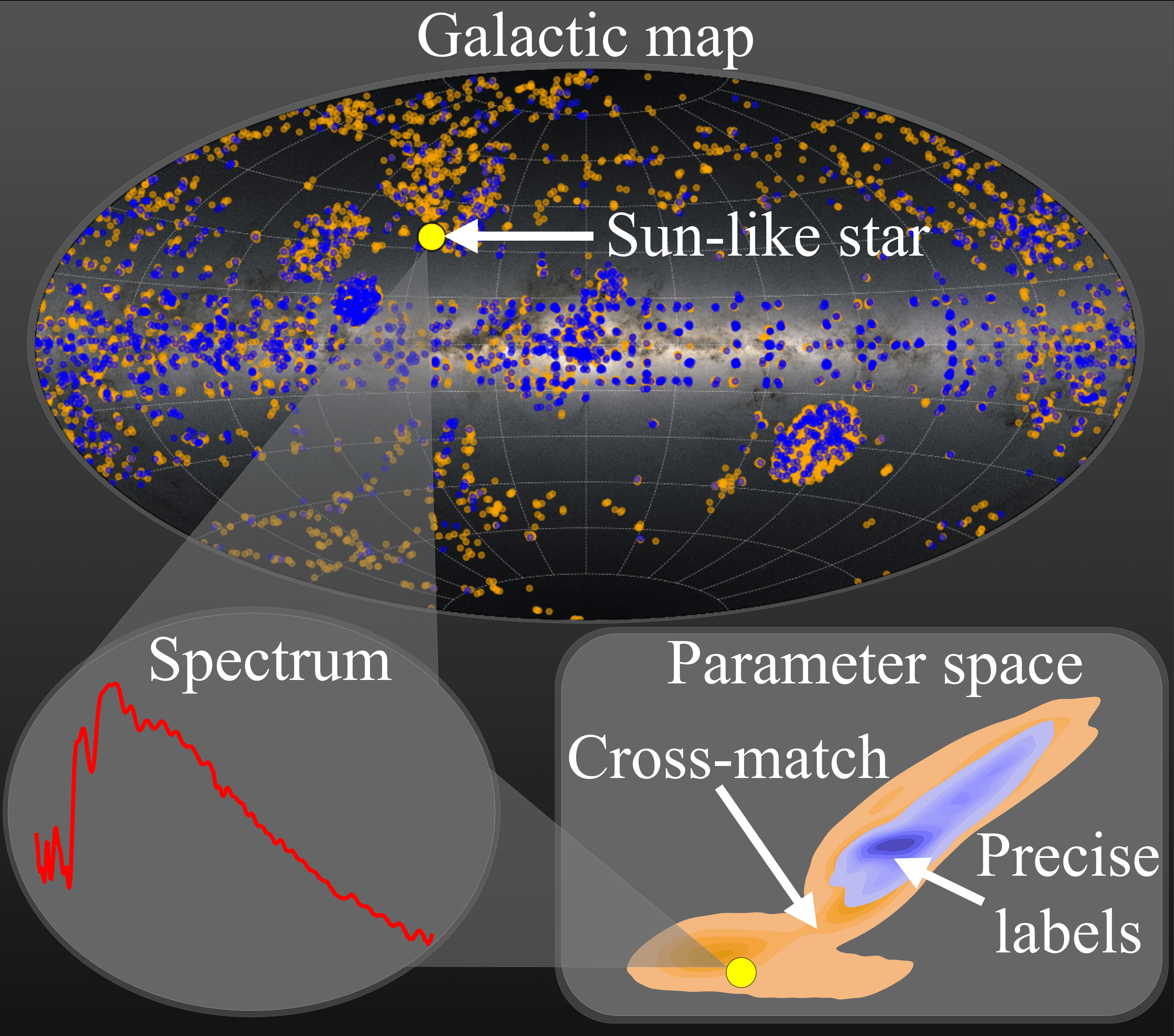} \\
                    \hline \small
                        Estimating stellar parameters with cross-matched \\ astronomical catalogs under selection bias
                } \\
                \bottomrule
            \end{tabular}
        }
    \end{center}
    \label{tab:case_studies}
\end{table}

\begin{itemize}
    \item Case study I reconstructs gamma rays to localize and measure astrophysical sources. 
    \item Case study II infers properties of Milky Way stars using two different galactic models.
    \item Case study III estimates stellar parameters with cross-matched astronomical catalogs under selection bias.
\end{itemize}
\noindent By connecting generative AI, classical statistics, and modern machine learning, our approach enables scientists to perform trustworthy inference using neural posteriors and generative models in inverse problems, even when train data differ from targets. While the examples in this paper are focused in the physical sciences, our framework can equally advance mathematically principled scientific discovery in biology, environmental science, medical research, industrial processes, and other fields where traditional methods fail.

\section{Methodology} \label{sec:methods}

\noindent This paper proposes a new framework for reliable scientific inference under intractable likelihoods, which bridges classical (frequentist) statistics \citepsupp{neyman_outline_1937,neyman_problem_1935} with Bayesian inference and machine learning. In this section, we describe the experimental set-up and give an overview of the FreB protocol. We refer the reader to {Appendix~\ref{sec:theory_and_algos} for theoretical details, proofs, and algorithms.

\subsection{Experimental set-up}\label{sec:experimental_setup}

\noindent  Suppose we have {\em unlabeled} target data $$ \mathcal{T}_{\rm target} = \left\{ (\theta_1^{*}, X_1^{\rm target}) \, \ldots  (\theta_{N}^{*}, X_{N}^{\rm target})\right\} \sim p_{\rm target}(\theta)  p(X | \theta), $$ where neither the true parameters $\theta_1^{*}, \ldots, \theta_{N}^{*}$ nor the distribution $p_{\rm target}(\theta)$ are known to the scientist.\footnote{From a classical statistics perspective, these parameters are perhaps best understood as ``latent variables.'' Although each parameter $\theta^*_i$ is {\em fixed} and not random for each object $i$, we define a marginal distribution for $\theta$ that represents its prevalence in the target population. In addition, in some applications we only observe each target object once (that is, the sample size $n=1$ for each parameter).} With generative models, the scientist learns an estimate of the posterior distribution $\widehat{\pi}(\theta|X)$, which represents the plausibility of parameters given data $X$. The modern approach for large-scale complex systems is to pretrain such models on broad data from different sources, or train models on synthetic examples from a physics-based simulator and chosen prior. More specifically, the posterior is learned using labeled train data $$ \mathcal{T}_{\rm train} = \{ (\theta_1, X_1) \, \ldots  (\theta_B, X_B)\} \sim \pi(\theta) \widehat{p}(X | \theta),$$ where both the prior distribution, denoted $\pi(\theta)$, and the assumed likelihood $\widehat{p}(X | \theta)$ can be different from $p_{\rm target}(\theta)$ and  $p(X | \theta)$, respectively, due to prior mismatch (Section~\ref{sec:Brutus_study}), selection effects (Section~\ref{sec:Gaia_study}), and model misspecifications with respect to $p(X|\theta)$ (so-called systematics; see Appendix~\ref{sec:well_specified_2D}).
\\
\\
\noindent FreB offers a practical means of adjusting and checking pre-trained posterior models against calibration data 
$$ \mathcal{T}_{\rm cal} = \{ (\theta_1', X_1') \, \ldots  (\theta_{B'}', X_{B'}')\} \sim r(\theta)  p(X | \theta),$$
where the reference distribution $r(\theta)$ covers the parameter space $\Theta$ of interest. The assumption is that the calibration data stem from the {\em same} physical process and likelihood $p(X|\theta)$ as the target data---but the distribution $r(\theta)$ does {\em not} need to be the same as $p_{\rm target}(\theta)$.
\\
\\
\noindent Our goal is to construct a confidence region $C(X)$ for $\theta$ that has correct frequentist coverage; that is, $\P_{X|\theta}(\theta \in C(X)) \geq 1 - \alpha$ for every unknown parameter $\theta$. Since the conditional distribution $X \mid \theta$ is assumed to be the same for calibration and target data, we have the result that if $C(X)$ ensures valid coverage for the calibration set, then it will also do so for our targets of interest. Note that posterior-based credible regions are generally not valid. For example, HPD level sets $H_c(X)=\left\{\theta:  \widehat{\pi}(\theta|X)> c \right\}$  do not usually have frequentist coverage properties in general settings, even when  $\widehat{\pi}(\theta|X) = \pi(\theta|X)$ and $\int_{H_c(X)} \pi(\theta|X) d\theta \geq 1-\alpha$.
\\
\\
\noindent In the next section, we present our protocol for ``reshaping'' a posterior (one form of  distribution) to a valid confidence procedure (another distribution) in general settings with, for example, intractable likelihoods, small sample sizes and misspecified models.

\subsection{A protocol for trustworthy scientific inference: from posteriors to locally valid confidence procedures}\label{sec:protocol}

\noindent Our proposed Frequentist-Bayes procedure mirrors the style of HPD level sets $H_c(X)=\left\{\theta:  \widehat{\pi}(\theta|X)> c \right\}$ in Bayesian inference, while providing frequentist coverage properties for every $\theta \in \Theta$, regardless of $\pi(\theta)$,  $\widehat{\pi}(\theta|X)$,  and the number of events per parameter. The main steps, summarized by the flow chart in Figure~\ref{fig:freB_branches} and illustrated by the 1D synthetic example in Figure~\ref{fig:1D_example}, are as follows:

\begin{enumerate}
    \item \textbf{Learn the posterior distribution}: From train data $\mathcal{T}_{\rm train}$, learn the posterior distribution ${\pi}(\theta|X)$ with, for example, a neural density estimator. The estimated posterior $\hat{\pi}(\theta|X)$, or a related function, is treated as a frequentist test statistic $\lambda(X; \theta)$. This statistic assigns a score $\lambda(X; \theta_0)$ that measures the degree to which a parameter value $\theta_0$ is plausible given that  $X$ is observed. Examples of other posterior-based scores include the Bayes Frequentist Factor (BFF; \citemain{dalmasso_likelihood-free_2024}) and the Waldo test statistics \citepmain{masserano_simulator-based_2023}. Note that in modern AI applications, the initial posterior $\widehat \pi(\theta|X)$ has already been pre-computed on abundant train or simulated data. The key FreB procedure is then to adjust these outputs as outlined below.
    \item \textbf{Reshape the posterior into p-value functions:} From calibration data $\mathcal{T}_{\rm cal}$, learn a family of monotonic transformations  $F(\cdot; \theta)$ of the test statistic $\lambda$ as follows (see Algorithm~\ref{algo:rejection_prob0} and Equation~\ref{eq:p-value} for details):
    
    \begin{pvaluefunctionbox}
\begin{enumerate}
  \item For each point in the calibration sample
  $ \mathcal{T}_{\rm cal} = \{ (\theta_1, X_1) \, \ldots  (\theta_{B'}, X_{B'})\}$,
  compute the test statistic $\lambda(X;\theta)=\hat \pi(\theta|X)$.
  \item Compute the indicator variable $Y:=\mathbb{I}(\lambda(X;\theta) < t)$ for a grid of values of $t \in \mathbb{R}$ at each sample point.
  \item Estimate the rejection probability
  $ F(t;\theta) := \pr_{X| \theta} \left( \lambda(X;\theta) \leq t \right)$
  via a regression of $Y$ on $\theta$ and $t$, which is monotonic in $t$.
\end{enumerate}
\end{pvaluefunctionbox}

The $F(\cdot; \theta)$ functions are effectively ``amortized p-value functions''  that allow the construction of confidence sets at all miscoverage levels $\alpha$ simultaneously; see Figures~\ref{fig:1D_example}b,~\ref{fig:2D_example}b,~\ref{fig:cs1_revised}c,~\ref{fig:brutus_1}c, and~\ref{fig:cs3-fig}c for some examples.  Alternatively, if one is only interested in confidence sets at a prespecified level $\alpha$  (as in our case studies),  then directly estimate ``critical values" for $\lambda$,  $F^{-1}(\alpha; \theta)$, at fixed $\alpha$ (Algorithm \ref{alg:estimate_cutoffs}).
   
    \item \textbf{Construct confidence sets}: Finally, compute Frequentist-Bayes sets $B_\alpha(X)$ by taking level sets of a transformation $F(\cdot)$ of $\hat{\pi}(\theta|X)$: $$B_\alpha(X) = \left\{ \theta \in \Theta \mid F(\hat{\pi}(\theta|X); \theta) > \alpha \right\} = \left\{ \theta \in \Theta \mid \hat{\pi}(\theta|X) > F^{-1}(\alpha; \theta) \right\}.$$ This computation is ``amortized'' with respect to $X$ in the sense that once we have learned the posterior distribution (Step 1) and the monotonic transformation (Step 2), no further training is needed for new $X$: we can just evaluate the confidence set $B_\alpha(X)$.
    \item \textbf{Check local coverage of constructed confidence sets}:  After building confidence sets, check that the actual coverage probability $\P_{X|\theta}(\theta \in \hat{B}_\alpha(X))$ for data $X$ generated at $\theta$ is indeed the same as the nominal value $(1-\alpha)$, for {\em every} $\theta$ in the parameter space. This check is not part of the construction of confidence sets per se, but provides the scientist with an independent diagnostic tool to assess her final results. See Algorithm~\ref{alg:estimate_coverage} for an efficient way to compute such diagnostics from held-out calibration data which we denote by $\T_{diag}$ in the flowchart. Figure~\ref{fig:1D_example}a-b, \textit{right}, illustrates how these diagnostics can help domain scientists identify regions of the parameter space where the confidence sets might under- or over-cover, even when parameter distribution of the target source is unknown.
\end{enumerate}

\noindent
In Appendix~\ref{sec:theory_and_algos}, we prove the following key properties of our framework, which are illustrated by the 2D example in Section \ref{sec:2D_example}, Figure~\ref{fig:2D_example}:
\begin{itemize}
    \item \textbf{Correct local coverage across the parameter space}: The Frequentist-Bayes confidence procedure achieves $(1-\alpha)100\%$ coverage for all parameter values regardless of the train distribution (when the universal set used for recalibration is large enough); see Figure~\ref{fig:2D_example}b, right, for a synthetic example. 
    \\ 
    \\
    Refer to Appendix~\ref{sec:validity} for theoretical results: specifically, see Theorem~\ref{thm:pval_right_coverage} for guarantees on validity of the p-value approach as the number of simulations $B'$ in the universal set  increases, Theorem~\ref{thm:pval_rate} for convergence rates, and Theorems~\ref{thm:valid_tests} and~\ref{thm:valid_tests_rate} for the corresponding results under the critical value approach.
    \item \textbf{Efficiency with well-specified models and no data set shift:} When the train and target distributions are the same, Frequentist-Bayes sets are optimal, with a smaller average size than other confidence sets with the same coverage properties; see  Figure~\ref{fig:2D_example}b, center, for a synthetic example.
    \\ 
    \\
    Refer to Appendix~\ref{sec:power} for theoretical results: specifically, see Theorem~\ref{thm:best_region} for a formal proof that, among all valid confidence sets, the Frequentist-Bayes set is the set that minimizes $\E \left[|A(X)|\right]$, where $|A(X)|$ is the size of a set $A$.
\end{itemize}

\begin{figure}[t!]
    \centering
    \includegraphics[trim={0 0 0 2cm},clip,width=0.8\columnwidth]{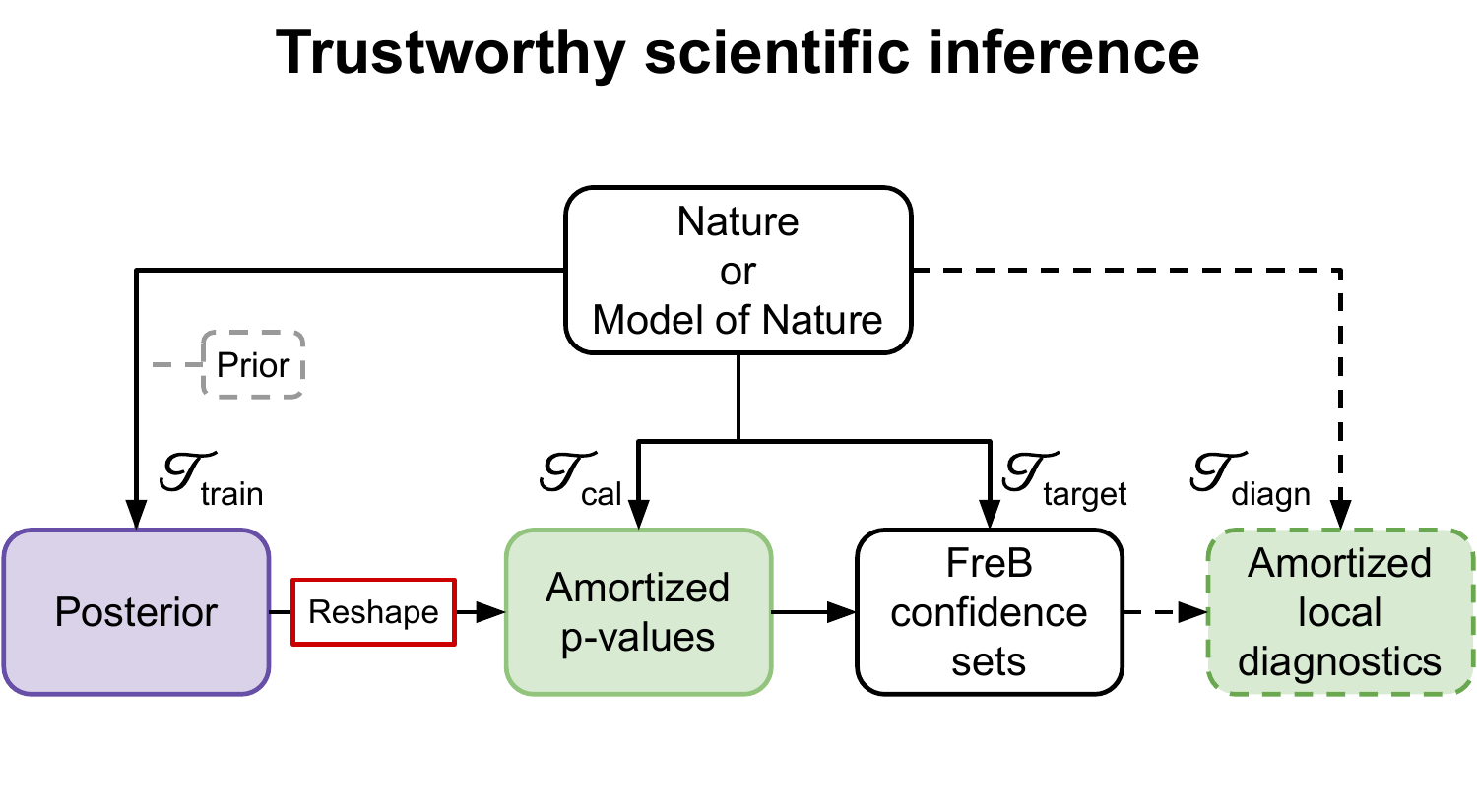}
    \caption{\small {\bf Flowchart of our protocol for trustworthy scientific inference.} The posterior can be derived using any method (AI, statistical or physics-based model); in this paper we use generative models to learn the posterior from train data $\mathcal{T}_{\rm cal}$. Our main method's contribution is proposing machine learning algorithms that efficiently compute (i) amortized p-values, and (ii) amortized local diagnostics; here shown as \textcolor{green}{green} boxes. Both computations are based on labeled examples, here denoted by $\mathcal{T}_{\rm cal}$ and $\mathcal{T}_{\rm diagn}$, respectively. The FreB confidence sets are computed on unlabeled target data, $\mathcal{T}_{\rm target}$. The local diagnostics branch (connected by dashed lines) represents an {\em independent} check of whether the final FreB confidence sets actually contain the true parameter with the stated probability, no matter what that parameter value is.}
    \label{fig:freB_branches}
\end{figure}

\begin{figure}[h!]
    \centering
    \includegraphics[width=0.95\linewidth]{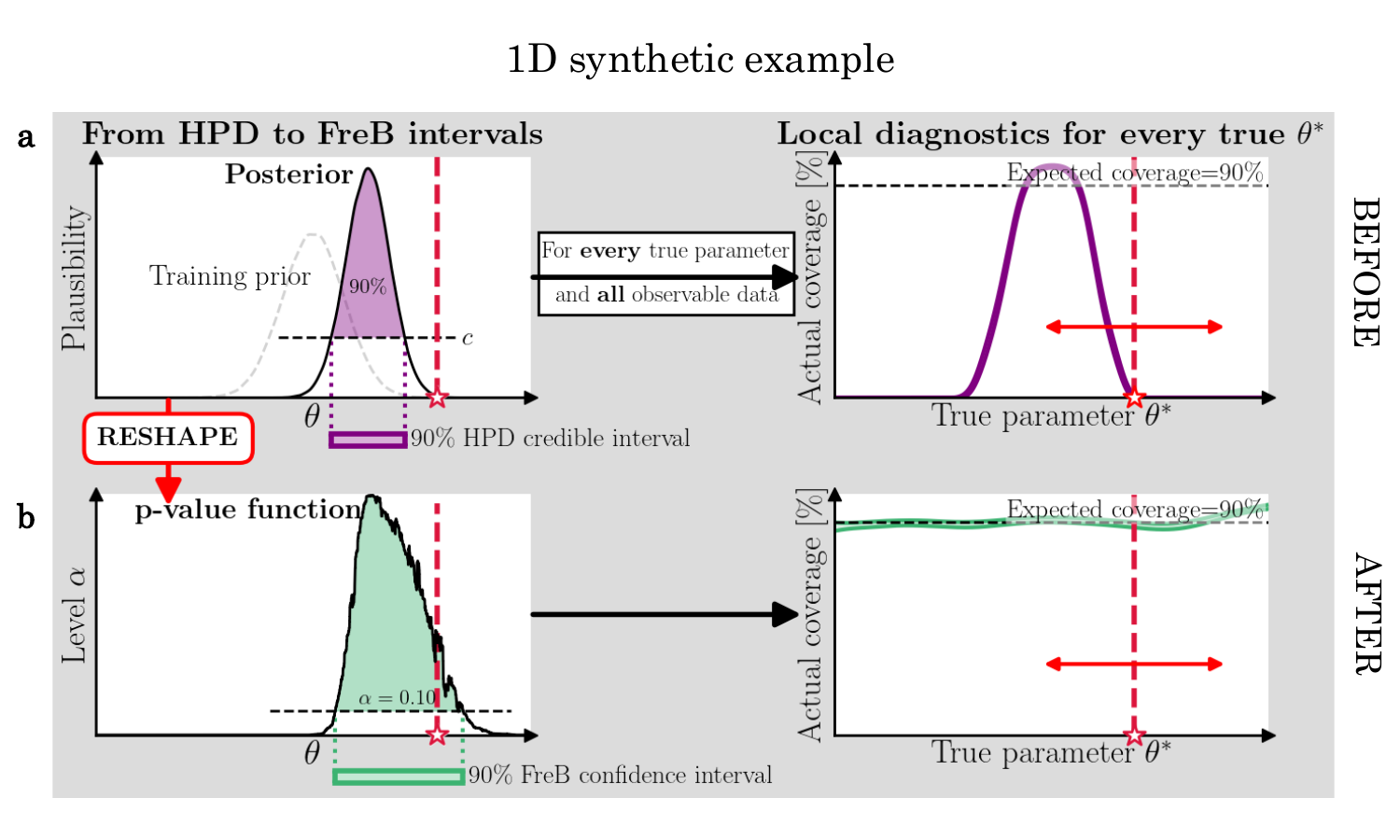}
    \caption{\small {\bf 1D synthetic example of the FreB protocol using masked autoregressive flow.} \textbf{Panel a:} 
    {\em Left,} The typical workflow for inferring parameters with neural density estimators and generative models is to learn the posterior from train data, then slice it to compute highest-posterior density (HPD) sets for new observations. The \textcolor{Purple}{purple} interval shows a 90\% HPD credible interval for an observation whose true parameter (\textcolor{red}{red} star) lies in the tail of the training prior. {\em Right,} Our diagnostic tool learns local coverage performance (that is, the empirical confidence level) for all scenarios of the truth using labeled examples. The diagnostic plot reveals that the actual chance (coverage probability, y-axis) that the HPD credible interval recovers the truth is far less than the expected coverage of 90\% for a wide range of values of $\theta^*$ (x-axis). \textbf{Panel b}: Performance after reshaping posteriors. {\em Left,} FreB reshapes the posterior density to a p-value function, which we then slice to obtain valid (``Frequentist-Bayes''; FreB) $(1-\alpha) 100$\% confidence intervals at $\alpha=0.1$ (\textcolor{Emerald}{green}). {\em Right,} The diagnostic plot indicates that the actual chance that FreB sets contain the true parameter value is close to the desired coverage probability for every instance of $\theta^*$.  Repeated observations at each $\theta^*$ are not required to learn FreB or the diagnostics---all computations are also ``amortized'': once learned for labeled examples, they can be deployed to new data without retraining.}  \label{fig:1D_example}
\end{figure}

\section{Results} \label{sec:results}

\subsection{2D synthetic example} \label{sec:2D_example}
\noindent We start with a 2D Gaussian mixture model example from the Bayesian simulator-based inference literature \citepmain{clarte_componentwise_2021,toni_approximate_2009,simola_adaptive_2021,lueckmann_benchmarking_2021} to illustrate the two key FreB properties described in Section~\ref{sec:methods}:
(i) FreB reshapes posteriors to confidence sets with nominal local coverage across the parameter space, and (ii) the confidence sets are efficient (with smallest average size) when train and target distributions are the same.\\

\noindent In this example, the true likelihood of the target data is given by a mixture of two normal distributions,
\begin{equation*}
    p(X\vert\theta) = \frac{1}{2}\mathcal{N}(\theta, I) + \frac{1}{2}\mathcal{N}(\theta, \sigma^2 I),
\end{equation*}
where $\sigma=0.1$, and the common mean $\theta$ is the parameter of interest.\\

\noindent We assume the posterior was learned using train data
 $$ \mathcal{T}_{\rm train} = \{ (\theta_1, X_1) \, \ldots  (\theta_B, X_B)\} \sim \pi(\theta) \widehat{p}(X | \theta),$$
with a localized prior $\pi(\theta) = \mathcal{N}(0, 2I)$ and a slightly misspecified forward model,
\begin{equation}\label{eq:misspecified_forward_model}
    \hat{p}(X\vert\theta) = \frac{1}{2}\mathcal{N}((1-\delta)\cdot \theta, I) + \frac{1}{2}\mathcal{N}((1-\delta)\cdot \theta, \sigma^2 I).
\end{equation}
with $\delta=0.25$. Figure~\ref{fig:2D_example} Panel a shows HPD sets from a flow matching estimator trained with $B=50{,}000$ such examples. The flow matching model is a good estimator of the posterior \citepmain{lipman_flow_2023}. Hence, it is not surprising that the inference results are good when the true $\theta^*$ is close to the center of the prior (``Well-aligned prior'', Panel a-\textit{center}). However, the model fails to provide valid inference for individual instances far from the center of the prior: Panel a-\textit{left} (``Misaligned prior'') indicates one such challenging case for a sample drawn from the likelihood at $\theta^*=(8.5, 8.5)$. More generally, the chance of covering the true $\theta^*$ with credible regions falls to zero as the true mean $\theta^*$ is increasingly further away from the center of the prior; Panel a-\textit{right} shows local coverage diagnostics.\\

\noindent Therefore we need to adjust the posterior to achieve reliable uncertainty quantification across the parameter space. With access to some additional data from the true data-generating process, we can reshape the posterior into valid confidence sets via the p-value function. In this example, we use calibration data
$$ \mathcal{T}_{\rm cal} = \{ (\theta_1', X_1') \, \ldots  (\theta_{B'}', X_{B'}')\} \sim r(\theta)  p(X | \theta),$$
with reference distribution $r(\theta) = \mathcal{N}(0, 36 I)$ over $\theta$. Figure~\ref{fig:2D_example} Panel b shows the FreB results from a monotone neural network that learn the p-value function from  $B'=30,000$ examples. As seen in Panel b-\textit{right}, inference results are valid across the parameter space: misaligned priors lead to larger confidence sets, while well-aligned priors yield tighter confidence sets.\\

\noindent Finally, as mentioned, posterior-based intervals do not guarantee valid inference even under idealized conditions with a well-specified forward model. Indeed, Figure~\ref{fig:2D_example_well_specified} in Appendix~\ref{sec:well_specified_2D} shows similar results with $\delta=0$; that is, when we train the flow matching model on data from the same data-generating process as the target data: Here, the FreB sets are even smaller than for the setting with a misspecified forward model. These results are consistent with Theorem~\label{thm:best_region} that states that frequentist-Bayesian procedures have optimal average power when train and target distributions are the same. That is, {\em FreB applied to posteriors (``before'') leads to valid confidence regions (``after''), and better alignment of train-test data in terms of both the prior and the forward model leads to higher constraining power and smaller regions on average.} \\\\

\begin{figure}[h!]
    \centering
    \includegraphics[width=1.0\columnwidth,trim=0 0 4cm 0,clip]{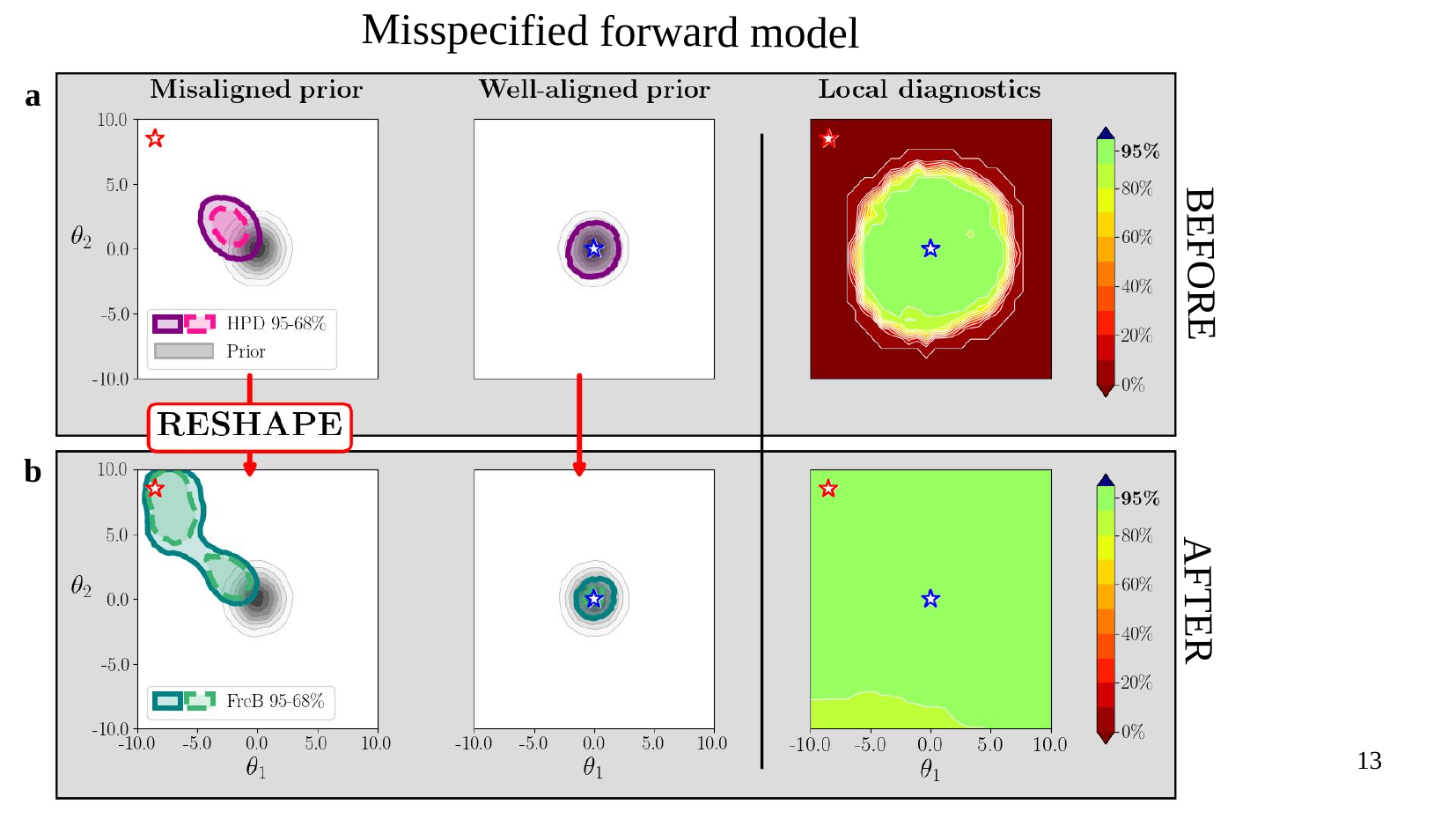}
    \caption{\small \textbf{2D synthetic example to illustrate the FreB protocol for a misspecified forward model and a localized prior.} The task is to infer the common mean $\theta$ of a mixture of two Gaussians with different covariances using a flow matching generative model trained with a localized prior centered at the origin. \textbf{Panel a:} 95\% and 68\% HPD sets for two scenarios where the prior is misaligned ({\em left}) versus well-aligned {(\em center}) with the true $\theta^*$  (\textcolor{red}{red} star). {\em Right,} Local diagnostics of 95\% HPD sets shows that the actual coverage of these sets can be very far from the nominal 95\% level, when the truth is further away from the center where the train data are concentrated. \textbf{Panel b:} After reshaping and slicing the posteriors as in Figure~\ref{fig:1D_example}b, we obtain the corresponding FreB sets. For all instances of $\theta$ and for all levels of $\alpha$, domain scientists can achieve the desired coverage level, here illustrated for the 95\% case in the {\em right} plot. That is, FreB sets are robust against misaligned training priors. The size of FreB sets is also smaller for well-aligned priors (compare {\em center bottom} plot with the {\em left bottom} plot). See Figure~\ref{fig:2D_example_well_specified} in Appendix~\ref{sec:well_specified_2D} for a similar example with a well-specified forward model.}
    \label{fig:2D_example}
\end{figure}

\subsection{Case study I: Reconstructing gamma rays to localize and measure astrophysical sources}
\label{sec:gamma_showers}

\noindent This case study illustrates how one can identify and reconstruct previously unknown astrophysical sources, which might be missed or misinterpreted if generative models are applied naively to infer key parameters of interest.
\\
\\
\noindent Gamma rays yield crucial information on violent phenomena (such as supernovas and black hole mergers) that take place in the cosmos. However, unlike most astronomical fields---from radio to x-ray astronomy---where photons are directly measured and their source direction can be traced back to their origin, tracing high-energy gamma rays requires an indirect approach. Earth's atmosphere is generally opaque to gamma rays, which can only be inferred from the cascades of secondary particles they create when they interact with the atmosphere (see Table~\ref{tab:case_studies}-I).
\\
\\
\noindent Therefore, a major challenge in high-energy astrophysics research is reconstructing properties of the original gamma ray (namely its energy and arrival direction) based on measurements of secondary particle types, spatial patterns, and arrival timing~\citepmain{chadwick_35_2021}; see Figure~\ref{fig:cs1_revised}a, \textit{left} and \textit{center}. This method of detection is further complicated by the fact that charged cosmic rays (e.g., protons or light nuclei), which are far more frequent, produce similar atmospheric showers of particles; see, e.g.,~\citemain{abreu_science_2025} for a discussion of the gamma-hadron separation challenge.
\\
\\
Here we consider the problem of estimating the parameter vector $\theta = (E, Z, A)$---representing the energy (E), zenith angle (Z), and azimuthal angle (A) of the incoming gamma ray---from simulated data $X$ that include the types of particles (electrons, photons, etc.), their count rate and density, and various properties (e.g., energy, direction) of secondary particles detected on the ground.
\\
\\
The generative model is trained with synthetic examples from an astrophysical source with the characteristics of the \textit{Crab Nebula}, a pulsar-wind nebula emitting the brightest and stable TeV signal in the northern hemisphere sky. The target data (gamma rays to be reconstructed)  originate from two astrophysical sources with the characteristics of:
\begin{itemize}
    \item \textbf{Markarian 421 (Mrk421)}, a well-studied blazar that is among the brightest known gamma-ray sources \citepmain{abdo_fermi_2011}; 
    \item \textbf{Dark Matter}, such as that expected from theoretical models of dark matter annihilation near the Galactic Center \citepmain{mukherjee_advances_2024,cirelli_dark_2024}.
\end{itemize}
All events are simulated using \texttt{Corsika} \citepmain{heck_corsika_1998} with an idealized detector that perfectly records all secondary particles reaching the ground. Their effective energy distributions are shown in Figure~\ref{fig:cs1_revised}a, \textit{right}. We learn the posterior distribution $\pi(\theta | X)$ by flow matching \citepmain{wildberger_flow_2023,lipman_flow_2023} and construct 90\% HPD and FreB sets for each event. When training with Crab Nebula data, we observe the following:
\begin{itemize}
    \item \textit{HPD sets miss target events that are rare relative to the parameters of the train examples.} The actual chance that a 90\% HPD set includes the true parameter is on average 86\% for the Crab Nebula, 81\% for Mrk421, and down to 73\% for gamma-rays originating from the DM signal. The poor performance on the DM source in particular is driven by a higher frequency of very-high-energy gamma rays like the ``rare'' 8.4 TeV event in Figure~\ref{fig:cs1_revised}b, {\em center}, resulting in a credible set biased toward lower energies. 
    \item \textit{The corresponding FreB sets correctly reflect constraining power.} We can reshape the {\em same} estimated posteriors from flow matching to create FreB sets with valid and informative uncertainties. In Figure~\ref{fig:cs1_revised}c, each individual FreB set is now at the 90\% nominal value regardless of the origin of the gamma ray. This adjustment allows us to reliably identify and reconstruct different astrophysical sources---like a Dark Matter annihilation signal---as long as we have labeled examples (calibration data) that follow the same physical process as the target.
\end{itemize}
  
\begin{figure}[p]
    \centering
    \includegraphics[width=1.0\linewidth]    {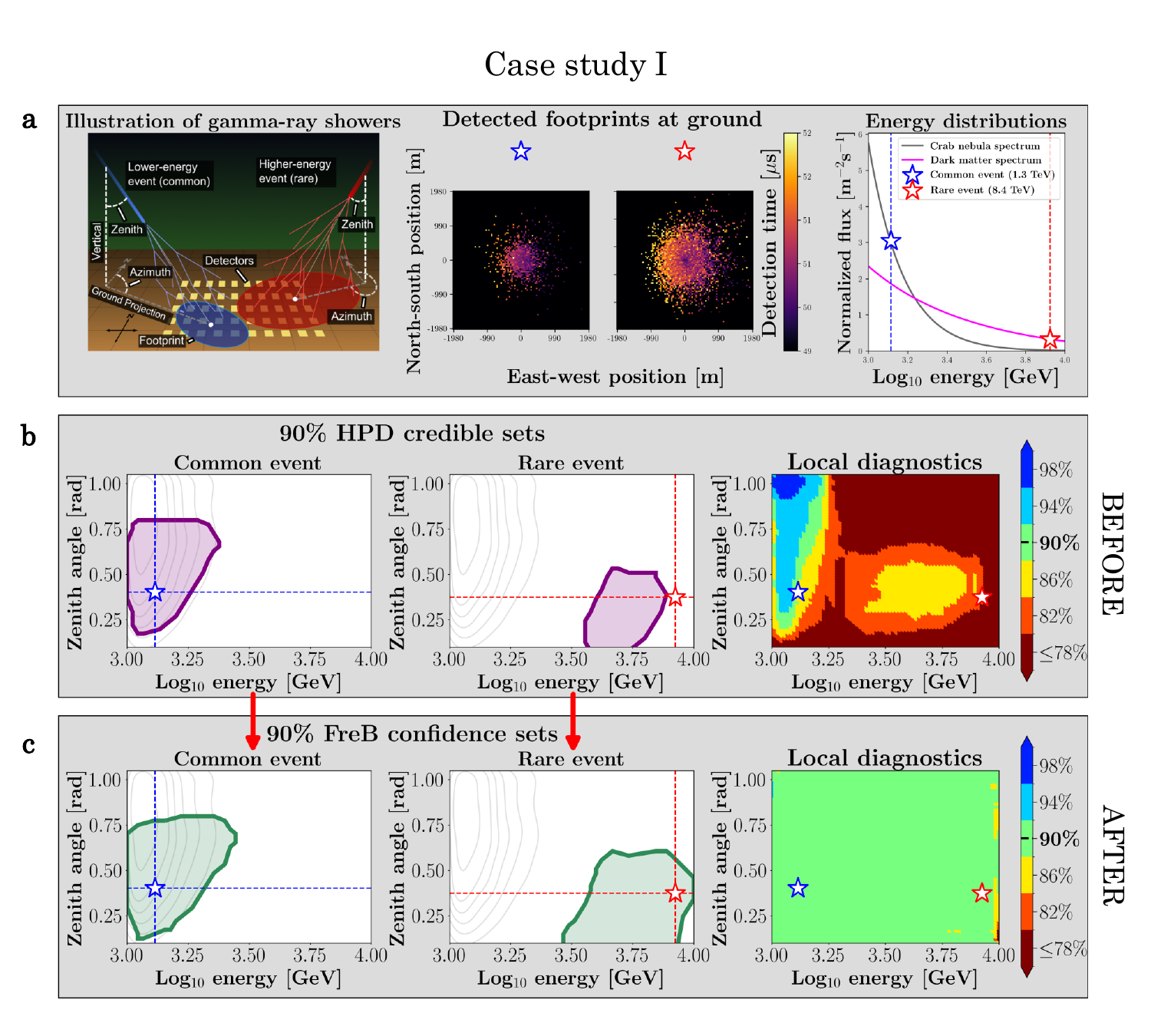}
    \caption{\small \textbf{Reliable reconstruction of gamma-ray sources.} 
    \textbf{Panel a:} \textit{Left,} Visual representation of atmospheric showers and their parameters of interest. \textit{Center,} Detected footprints at ground level for two example events. \textit{Right,} Distribution of gamma-ray energies for the Crab Nebula (training data source) and Dark Matter (possible target source).  Gamma rays at lower energies (e.g. the example in \textcolor{blue}{blue}) are more commonly observed for the Crab Nebula than for the Dark Matter source, whereas gamma rays at higher energies  (e.g. the example in \textcolor{red}{red}) are rare for the Crab Nebula relative the Dark Matter source. \textbf{Panel b:} Parameter estimates when learning posterior with Crab Nebula data. {\em Left and center,} 90\% HPD sets for the common and the rare event. The estimates for the rare event are biased towards lower energies; the credible region has an actual coverage that is smaller than what is expected. {\em Right,} Diagnostics plot of local coverage of 90\% HPD sets reveals undercoverage, especially at higher energies. \textbf{Panel c:} Parameter estimates after reshaping the posteriors. {\em Left and center,} The adjusted 90\% FreB sets provide valid and informative uncertainty. {\em Right,} Local diagnostic plot confirms that 90\% FreB sets are  uniformly valid across the parameter space. (The azimuthal angle is not shown in the figure)}
    \label{fig:cs1_revised}
\end{figure}

\subsection{Case study II: Inferring properties of Milky Way stars using two different galactic models} \label{sec:Brutus_study}

In this case study, we show how different models (priors) of nature can lead to seemingly conflicting scientific conclusions when using generative AI---an apparent paradox which FreB can resolve under the assumption that the data used to learn the FreB transformations encode the same likelihood as the target.
\\ \\
Galaxies are formed through a complex process of hierarchical merging and assembly, with stars migrating from star clusters, which combine to form small galaxies, and which then merge to make galaxies such as our own Milky Way. Recovering the exact positions of stars, their motions through the sky, and their ages and chemical compositions allows us to reconstruct the structure, evolution, and assembly history of the Milky Way as well as the universe beyond \citepmain{deason+24}. These discoveries have traditionally been made by measuring stellar \textit{spectra}---``fingerprints'' of emitted light across different wavelengths---with the unprecedented depth and breadth of next-generation instrumentation, such as DESI \citepmain{koposov_desi_2024} (see also Case Study III). However, these surveys traditionally can only target the brightest $<1\%$ of stars visible through imaging. Using \textit{photometry}---the brightness of a star in images taken at different wavelengths---therefore opens up the ability to do much more comprehensive studies of Galactic structure and formation at the cost of individual sources having larger parameter uncertainties \citepmain{green+19,anders+22,speagle+24}.
\\
\\
\noindent Analyses of stellar photometry often start with a Galactic model, which describes the galaxy's stellar population and a forward model (likelihood), which maps stars to their expected evolutionary parameters and associated observables according to physical theory. A typical Galactic model consists of three components: a ``thick disk'', a ``thin disk'', and a ``stellar halo''. Each component represents a subpopulation of objects which together capture much of the Milky Way Galaxy's structure. This structure can be summarized in terms of the empirical age-metallicity relationship implied by the mixture of galactic components, as rendered in Figure~\ref{fig:brutus_1}a-\textit{left}. When a new star is identified, the evolving mixture of these components along the star's line-of-sight then naturally induces a prior distribution for that star's properties.
\\
\\
We focus on five key stellar properties that define $\theta=(\log g, T_{\text{eff}}, [Fe/H]_{\text{surf}}, \log L, \log d)$. This parameter includes the star's (log) surface gravity, effective temperature, surface metallicity (i.e. overall chemical enrichment relative to our sun), (log) luminosity, and its (log) distance from the Sun. (Refer to the online supplement on case study II for the true parameter values.) Our priors are derived according to stellar evolution theories using \texttt{brutus} \citepmain{speagle_deriving_2025}, an open-source Python package tailored for fast stellar characterization. The simulated photometry $X$ replicate the photometric bandpasses found in the 2MASS \citepmain{skrutskie_two_2006} and Pan-STARRS1 \citepmain{panstarrs_2016} surveys which span wavelengths in the near-infrared and optical, respectively.
\\
\\
We propose two competing models of our Milky Way galaxy:
\begin{itemize}
\item  \textbf{Model H (halo-biased)} increases the contribution of the halo by extending the metallicity range for stars in the Milky Way's periphery beyond typical models; e.g.~\citemain{anders+22,speagle_deriving_2025}. As the halo is generally comprised of older stars accreted from other small galaxies, this expanded model allows a greater chance that this new star could be associated with more recent halo accretion events.
\item \textbf{Model D (disk-biased)} diminishes the contribution of the halo, instead emphasizing objects typically found within the Galactic thin and thick disks. As the disk components are generally comprised of younger stars that have formed much more recently within the Milky Way (i.e. are not accreted), this model makes stronger assumptions about this new star originating from within our Galaxy.
\end{itemize}
Figure~\ref{fig:brutus_1}a-\textit{right} shows some of the pairwise marginals of the priors induced by these models.
These models are used to label a newly discovered stellar object at the Galactic sky coordinates $(\ell, b) = (70^\circ, 30^\circ)$. We estimate the posteriors, $\pi_H(\theta\vert X)$ and $\pi_D(\theta\vert X)$, with masked autoregressive flows \citepmain{papamakarios_masked_2017,tejero-cantero_sbi_2020} and construct 90\% HPD and FreB sets using priors $\pi_H(\theta)$ and $\pi_D(\theta)$, respectively. Appendix~\ref{sec:diagnostics} describes local diagnostics. In this case study, we observe the following:
\begin{itemize}
    \item \textit{HPD sets show stark disagreement for different galactic models, and with the true parameter.} For instance, under Model D, the estimated posterior $\hat{\pi}_D(\theta \mid X)$ of the example star (whose true parameter value is indicated with a red marker in Figure \ref{fig:brutus_1}b) significantly overestimates $[Fe/H]_{\text{surface}}$. Even Model H's posterior fails diagnostic tests, with HPD sets that rarely include all five stellar properties at once; refer to the online supplement for this case study for local coverage.
    \item \textit{FreB sets resolve the apparent paradox between different galactic models while ensuring nominal 90\% coverage of the true parameter.} Figure~\ref{fig:brutus_1}c displays cross-sections of the FreB sets which simultaneously include all five stellar properties. Appendix~\ref{sec:power} provides further insight into FreB sets' statistical power when good prior information is available.
\end{itemize}

\begin{figure}[p]
    \vspace*{-0.5in}
    \centering
    \includegraphics[width=0.98\linewidth]{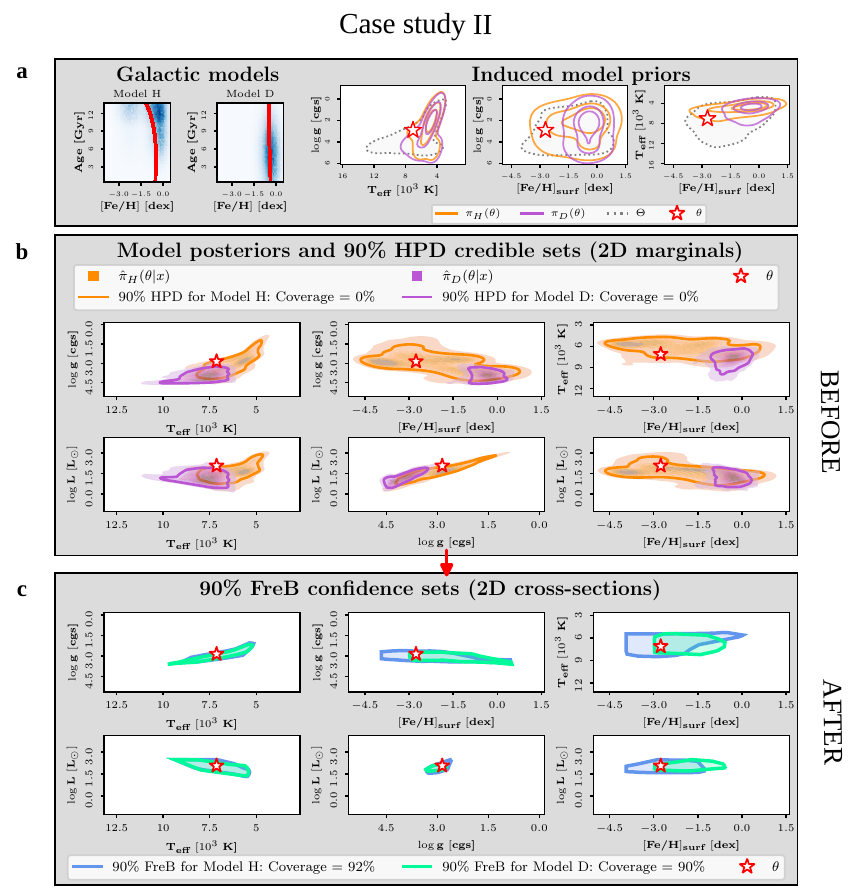}
    \caption{\small \textbf{Resolving tension between differing galactic models.} 
    \textbf{Panel a:} {\em Left,} The age-metallicity relationships implied by two galactic models. The \textcolor{red}{red} curves indicate the average internal metallicity for different ages. \textbf{Panel a:} {\em Right,} Surface-level priors induced by the galactic models along line of sight $(\ell, b) = (70^\circ, 30^\circ)$. Log surface gravity ($\log g$), effective temperature ($T_{\text{eff}}$), and surface metallicity ($[Fe/H]_{\text{surf}}$) are shown. The true parameter for an example star is marked in \textcolor{red}{red}, unknown at inference time. \textbf{Panel b:} Tension between Models \textcolor{orange}{H} and \textcolor{magenta}{D}'s posteriors at $X \sim p(X\vert\theta)$. Solid contours for each model show 90\% credible regions of highest posterior density, marginalized. The HPD regions have $0\%$ local coverage. Note that stellar distance has been marginalized for display clarity. \textbf{Panel c:} 90\% FreB sets for $\theta$ for Models \textcolor{orange}{H} and \textcolor{magenta}{D}. Each subplot shows cross-sections of the FreB sets at the true parameter. Local coverage for each FreB set is close to the nominal $90\%$ level.}
    \label{fig:brutus_1}
\end{figure}

\subsection{Case study III: Estimating stellar parameters with cross-matched astronomical catalogs under selection bias} \label{sec:Gaia_study}

\noindent In this case study, we go beyond using simulated data to demonstrate how our framework can handle observational studies with selection bias.
Using labeled examples, we adjust initial models pre-trained on survey data that suffer from selection bias and systematics.
\\
\\
\noindent Selection bias is a prevalent issue across various scientific fields, particularly in astronomical surveys, because of observational limitations and cost considerations.
For example, large-scale astronomical flagship surveys such as \textit{Gaia} \citepmain{gaia_collaboration_gaia_2023} and the Sloan Digital Sky Survey (SDSS, \citemain{almeida_eighteenth_2023}), and soon the Rubin Observatory Large Survey of Space and Time (LSST, \citemain{ivezic_lsst_2019}), do not uniformly observe (i.e., randomly sample) sources (e.g., stars and galaxies) across the sky due to complicated sampling mechanisms and systematics; see Section 2 of \citemain{tak_six_2024}.
Additionally, these surveys can only observe the brightest sources due to instrumental limitations, leading to further survey incompleteness and biased sampling of the underlying population \citepmain{malmquist_1922,malmquist_1925}. Furthermore, the vast majority of these sources are photometrically observed, with only a small subset followed up with higher-resolution spectroscopic measurements---such as the Sun-like star Figure~\ref{fig:cs3-fig}a-\textit{right}---that can be used to more precisely determine the properties of these sources; that is, provide more precise labels.\\
\\
\noindent Here we illustrate the challenge of {\em data set shift} due to selection bias---the phenomenon that the train data deviate significantly in distribution from the targets of interest because of observation limitations and label systematics \citepmain{laroche_closing_2025}---and how FreB can use data from follow-up surveys to ensure trustworthy inference in the presence of model misspecifications.\\
\\
\noindent Estimates of stellar parameters---e.g., surface gravity $\log g$, effective temperature $T_{\text{eff}}$, and metallicity $[Fe/H]$---are used in studies aimed at answering fundamental questions in astrophysics, from modeling stellar evolution \citepmain{minchev_estimating_2018} to understanding galaxy formation \citepmain{lagarde_deciphering_2021}.
In this case study, using a cross-match of stellar labels from APOGEE Data Release 17 \citepmain{majewski_apache_2017} and stellar spectra from Gaia Data Release 3 \citepmain{gaia_collaboration_gaia_2023}, we estimate the parameter vector $\theta = (\log g, \; T_{\text{eff}}, \; [Fe/H)])$ from data $X$ consisting of 110 Gaia BP/RP spectra coefficients \citepmain{gaia_collaboration_gaia_2023,laroche_closing_2025}. These coefficients trace extremely low-resolution spectral data more similar to imaging data than traditional high-resolution spectroscopy from surveys such as APOGEE.\\
\\
\noindent  We perform this estimation task in two data settings (see Figure~\ref{fig:cs3-fig}a for details):
\begin{itemize}
 \item \textbf{No selection bias}: the initial model is pre-trained on labeled data with the same distribution as the target stars of interest.
 \item \textbf{Selection bias (data set shift)}: the initial model is pre-trained on labeled data that are primarily larger, brighter giant branch (GB) stars, where APOGEE measurements are most precise, which are different from the target stars of interest, primarily smaller, fainter main sequence (MS) stars like our Sun along with low-metallicity stars. 
\end{itemize}

\noindent The setting with no selection bias includes training examples representing the full range our parameter space, as seen in the Kiel diagram in \ref{fig:cs3-fig}a-\textit{left}.
As a ``proof-of-concept'', we censor the remaining data to reflect a scenario where training data in the target region of parameter space are missing due to instrumental limitations.
We then assume that the censored data are later collected in a targeted follow-up survey and used to diagnose and adjust the initial posterior model.
This censoring pattern is depicted in Figure~\ref{fig:cs3-fig}a-\textit{middle}.
In our case, we estimate the posterior distribution $\pi(\theta| X)$ with masked autoregressive flows \citepmain{papamakarios_masked_2017,tejero-cantero_sbi_2020} and construct 90\% HPD and FreB sets in both data settings with and without selection bias (see the online supplement for case study III for details). 
More generally, our initial model could be purely based on synthetic data from a physics-based simulator, like Prospector \citepmain{johnson_stellar_2021}, or it could represent a large ``foundation'' model pre-trained on broad data, like SpectraFM \citepmain{koblischke_spectrafm_2024}.\\
\\
\noindent In this case study, we observe the following: 
\begin{itemize}
    \item \textit{FreB enables valid and precise stellar parameter estimation when selection effects and label systematics are minimized} (see Figure~\ref{fig:cs3-fig}b).
    Without model misspecification (prior and likelihood), HPD credible sets have high constraining power. They have correct (marginal) coverage if one \textit{averages} over the entire parameter space, but each HPD set undercovers in parameter regions that are underrepresented in the labeled set (c.f., Appendix~\ref{sec:confidence_procedures}).
    After reshaping posteriors, local coverage is ensured across the entire parameter space, while maintaining tight parameter constraints.\\
    \item \textit{FreB provides reliable parameter constraints and interpretable diagnostics even under selection effects and systematics} (see Figure~\ref{fig:cs3-fig}c).
    With a model pre-trained primarily on GB stars, there is a near 0\% chance that traditional HPD sets contain the true parameter of a MS or metal-poor star, which would fall outside of the bulk of the train data with respect to the underlying parameters (see the online supplement for this case study for further details). 
    However, by reshaping posteriors with a follow-up survey and FreB, we can ensure the desired local coverage across the entire parameter space, albeit with larger uncertainties in parameter regions that are underrepresented in the train data.\\
\end{itemize}

\begin{figure}[p]
    \vspace*{-0.5in}
    \centering
    \includegraphics[width=1.0\linewidth]{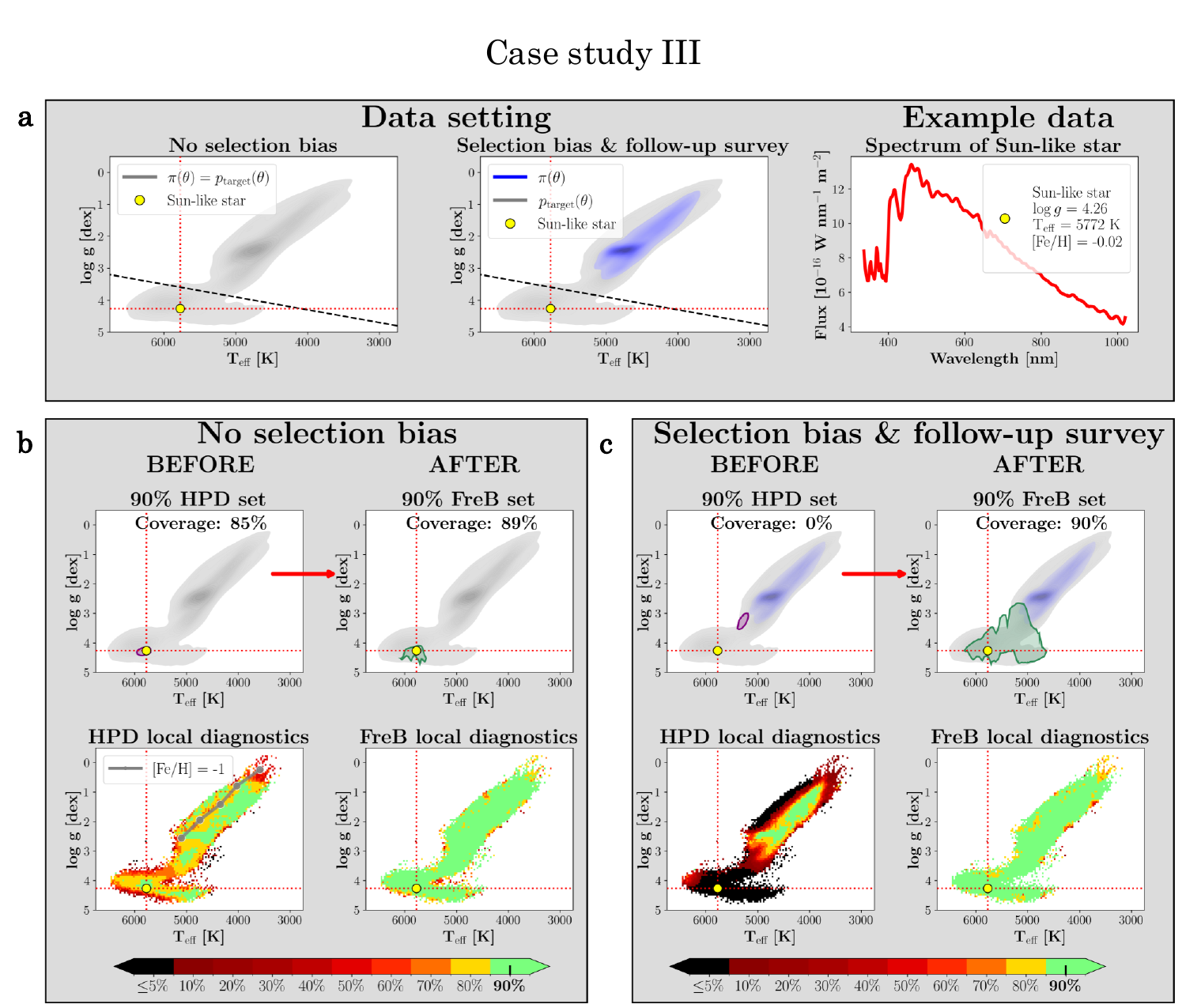}
    \caption{\small \textbf{Reshaping and diagnosing pre-trained models with labeled data.}
    \textbf{Panel a:} Kiel diagrams displaying the training distribution of stellar surface gravities $\log g$ against the corresponding effective temperatures $T_{\text{eff}}$ for two data settings, where the labeled data have the same distribution of the target stars of interest ({\em left}, ``No selection bias''), and where the labeled data, primarily GB stars, are different from the target stars, primarily MS and low-metallicity stars ({\em center}, ``Selection bias \& follow-up survey''). 
    The marked \textcolor{Dandelion}{yellow} dots represents the true stellar parameter for an example Sun-like target star, with its spectrum plotted ({\em right}, ``Example data'').
    \textbf{Panel b:} Under no selection bias, HPD sets have the desired 90\% coverage on average for the entire target population, but actual coverage for individual stars can vary.
    For example, credible regions for stars along the shown evolutionary metallicity track (\textcolor{gray}{gray}; [Fe/H]= -1.0 dex) tend to be too small ({\em left}, ``Before'').
    After reshaping posteriors, FreB sets all contain the true parameter with 90\% probability. 
    The sets still have the same high constraining power; that is, they are small in size like for the Sun-like example star ({\em right}, ``After'').
    \textbf{Panel c:}  However, under ``proof-of-concept'' censoring that reflects possible selection effects, HPD sets have a near-0 chance of including the true parameter for low-metallicity stars and main sequence stars like the Sun-like star ({\em left}, ``Before''). After reshaping posteriors with the censored data representing a ``follow-up survey'', FreB sets more accurately reflect the true uncertainty in the labels: the sets are larger but the chance of each set including the true parameter is now at the desired 90\% probability
    ({\em right}, ``After''). See the supplement on case study III for more details.}
    \label{fig:cs3-fig}
 \end{figure}

\section{Conclusions}
A direct application of generative models (neural posterior inference) can fail in two critical areas---biased estimates and lack of local validity---which lead to misleading scientific conclusions, even in ideal conditions with an all-knowing simulator and correctly labeled train data. These limitations produce a need for methods that ensure trustworthy scientific inference with generative models, as such models are increasingly used for inference tasks across the sciences. Here, we have presented a general, amortized procedure for transforming estimated posteriors into statistically valid Frequentist-Bayes confidence sets. FreB sets contain the true parameters with the desired probability regardless of what the true parameter values are, as long as we have a set of labeled examples (calibration data) from the same data-generating process (likelihood) as the target data. Moreover, if the domain scientist has good prior knowledge and a well-specified forward model, or equivalently, is able to collect train data from a distribution aligned with the target data, then FreB sets will return tighter parameter constraints than any other valid procedure, including procedures that do not use prior distributions.
\\
\\
\noindent Our Frequentist-Bayes protocol addresses the limitations of neural posterior inference and is a readily usable tool in a variety of experimental settings. When simulators are used for inference, FreB can ensure valid results even when faced with model misspecifications. Furthermore, it is effective in observational studies with partially labeled data, often challenged by selection effects and systematics.
\\
\\
\noindent As scientific data sets rapidly grow and physics-based models become increasingly complex, our FreB protocol is an advancement in ensuring that broadly used and state-of-the-art generative models are more trustworthy for scientific inference by providing the statistical foundations and diagnostics needed for accurate uncertainty quantification.
Future directions include developing a mathematically principled and physically grounded framework that integrates multi-instrument and multi-modal data to optimally constrain primary parameters of interest.
We envision incorporating our method into pipelines that use foundation models \citepmain{bommasani_opportunities_2022} by applying the FreB protocol after fine-tuning foundation models for specific use cases.
A related opportunity is detector optimization, and understanding how to tune instrument parameters for different observations.

\section{Discussion of relation to other methods}

\subsection{Classical statistical inference and approximate likelihood methods.}

\noindent FreB builds on the classical construction of confidence sets via inversion of hypothesis tests, which dates back to Neyman's seminal work \citepmain{neyman_problem_1935}. While this method has a long-standing tradition in scientific inference, it initially required tractable likelihoods and closed-form critical values, limiting its applicability. More recent advancements, especially within high-energy physics (HEP), have extended the Neyman construction to likelihood-free inference (LFI) scenarios \citepmain{feldman_unified_1998,cowan_asymptotic_2011,cranmer_practical_2015,schafer_constructing_2009}. These pioneering efforts highlighted critical open problems, such as efficiently constructing Neyman confidence sets in general settings, evaluating coverage without prohibitive computational costs, and effectively implementing hybrid statistical techniques \citepmain{cousins_lectures_2024,cousins_treatment_2006}. Building upon these foundations, several recent machine-learning-based techniques approximate the likelihood-ratio test (LRT) statistic and rely on asymptotic $\chi^2$ cutoffs to form confidence sets \citepmain{cranmer_approximating_2015}. While these approaches have shown promising performance in particle collider physics, the same methods can struggle with small-sample sizes or irregularities introduced by complex likelihoods \citepmain{algeri_searching_2020} and numerical estimation errors.
\\
\\
To address these limitations, \citemain{dalmasso_confidence_2020} developed \texttt{ACORE}, a method that estimates LRT cutoffs with machine learning techniques without resorting to asymptotic approximations, hence improving performance in limited-sample settings. Subsequently, \citemain{dalmasso_likelihood-free_2024} proposed Likelihood-Free Frequentist Inference (LF2I) as a modular framework of Neyman's inversion for likelihood-free inference and diagnostics, generalizing the approach to any test statistic. Other LF2I works based on approximate likelihoods include e.g. \citepmain{the_atlas_collaboration_implementation_2025,al_kadhim_amortized_2024}. FreB also falls under the general umbrella of LF2I but derives confidence sets directly from estimates of posterior distributions: the choice of a posterior test statistic allows the practitioner to take advantage of recent advances in the generative AI literature, as well as potentially leverage good prior knowledge to construct valid {\em and} small confidence sets (Appendix~\ref{sec:power}; \citemain{carzon_focusing_2025}).
\\
\\
\noindent More traditional techniques in the LFI literature that are based on posterior estimates usually fall under Approximate Bayesian Computation (ABC) methods. While ABC techniques have been very popular in many scientific fields---see for example \citemain{beaumont_bayesian_2004,beaumont_approximate_2010,sunnaker_approximate_2013}---they do not guarantee that the resulting credible regions are valid or precise.\\

\subsection{Bayesian SBI and conformal inference}

\noindent Recent advancements in simulation-based inference (SBI) have primarily come from cross-pollination with the machine learning literature \citemain{cranmer_frontier_2020,burkner_simulations_2025}. Several works have proposed learning algorithms that leverage novel neural density estimators such as normalizing flows (e.g., \citemain{papamakarios_fast_2016,lueckmann_flexible_2017,papamakarios_masked_2017,miller_truncated_2021,radev_jana_2023}), diffusion models (e.g., \citemain{geffner_score_2022,sharrock_sequential_2024,linhart_l-c2st_2023}), flow matching (e.g., \citemain{wildberger_flow_2023,holzschuh_flow_2024}) and consistency models (e.g., \citemain{schmitt_consistency_2024}). These methods have enabled a revolution in the inference capabilities available to domain scientists, but are not equipped with the statistical guarantees required by the rigor of the scientific method, as shown in, e.g., \citemain{hermans_crisis_2022} and \citemain{dalmasso_likelihood-free_2024}. The work of \citemain{delaunoy_towards_2022} successfully alleviates this issue by enforcing a balancing condition that yields more conservative posteriors, resulting in highest-posterior-density regions with approximate {\em average} coverage. Nonetheless, a posterior estimator that largely under-covers in some regions of the parameter space and largely over-covers in other regions would still be considered valid under the notion of average coverage. Our FreB work targets the stronger notion of validity defined in Equation~\eqref{eq:validity}, which ensures {\em local} coverage across the entire parameter space.
\\
\\
\noindent  Several methods have also been proposed to assess whether an estimated posterior distribution is consistent with the true posterior implied by the prior and likelihood \citemain{zhao2021diagnostics,linhart_l-c2st_2023,lemos2023sampling}. In addition, some work recalibrates the posterior when inconsistencies are found \citemain{dey2022calibrated}, and recent papers even adjust the prior  if given calibration data from the true joint distribution over {\em both} parameters and observable data \citemain{wehenkeladdressing,ruhlmann2025flow}. However, note that the above-mentioned simulation-based calibration (SBC) or “posterior calibration” methods differ from FreB. FreB explicitly aims to guarantee frequentist coverage of the true latent parameters: $\Pr_{X|\theta}(\theta \in C(X)) \ge 1 - \alpha$, no matter the value of $\theta$. This property ensures that coverage holds for all (unknown) parameter values, even when the prior is poorly chosen or the likelihood is misspecified — as long as the calibration sample accurately reflects the  conditional distribution of $X$ given $\theta$. Even perfectly estimated posteriors do not generally ensure this form of coverage.
\\
\\
\noindent Besides SBI-specific techniques, conformal methods have also become extremely popular in the machine learning community and beyond. Although conformal methods were originally developed for predictive problems, they can also enhance the marginal coverage properties of approximate Bayesian methods (see, e.g., \citemain{baragatti_approximate_2024} and \citemain{patel_variational_2023}). However, they do not guarantee frequentist (local) coverage across all parameter values. 

\subsection{\texttt{WALDO} and prediction-powered inference}

\noindent Several studies have used prediction methods on simulated datasets for inference on real observations, often without incorporating the necessary corrections to ensure valid uncertainty quantification (e.g., \citemain{dorigo_deep_2022,gerber_fast_2021,ho_approximate_2021}). To address this issue, \citemain{masserano_simulator-based_2023} introduced \texttt{WALDO}, a method that can take predictions from any machine learning algorithm and transform them into confidence sets with frequentist guarantees. Our FreB approach differs in that we estimate the full posterior distribution from simulated data rather than just point predictions, allowing us to derive confidence sets that are typically smaller and more accurate than those obtained through \texttt{WALDO}, particularly in cases where the posterior is multimodal or asymmetric.\footnote{For some illustrative examples, see the online supplement on Waldo vs FreB, which can be found under supplementary material in Section~\ref{sec:supplementary_material}. The results are consistent with the theoretical result on the optimality of Frequentist-Bayes sets in Appendix B.7.}
\\
\\
\noindent Prediction-powered inference \citepmain{angelopoulos_prediction-powered_2023} has also emerged as a promising framework that leverages both labeled training data $(X_1, Y_1),\dots,(X_n, Y_n)$ and additional unlabeled covariates $X_{n+1},\dots,X_{n+m}$ to enhance inference. However, this approach fundamentally differs from our setting, as its primary goal is to infer global parameters characterizing the data-generating process of the entire set, rather than constructing confidence sets for individual instances.

\subsection{Bayesian-frequentist approaches}

\noindent The interplay between Bayesian and frequentist methodologies has been explored in various contexts. \citemain{good_bayesnon-bayes_1992} proposed using the Bayes Factor as a frequentist test statistic, but only in scenarios where likelihoods are tractable. Similarly, \citemain{pratt_length_1961,yu_adaptive_2018,hoff_bayes-optimal_2023} showed that, when the likelihood is available, confidence sets derived from posterior distributions tend to be more efficient (in terms of expected volume) than those based purely on likelihood ratios. Our work extends these results to LFI settings, where likelihoods are intractable and confidence sets are constructed from posterior estimates obtained via generative models.
\\
\\
\noindent In addition, \citemain{wasserman_frasian_2011,fong_conformal_2021} showed that conformal inference can be applied to Bayesian models to construct prediction sets with valid frequentist coverage. Concretely, in that setting, one models the Bayesian predictive distribution $Y_{n+1} \mid X_{n+1}, (X_{n}, Y_n), \ldots, (X_{1}, Y_1)$ starting from a statistical model for $Y \mid \theta, X$. However, as previously mentioned, conformal methods only guarantee marginal coverage over $\theta$, which does not imply valid confidence sets for every parameter value. As a result, conformal procedures that exhibit severe under-coverage in some regions and strong over-coverage in others might still satisfy conformal guarantees, but would fail within our setting. In contrast, FreB provides confidence sets that maintain instance-wise validity  across the entire parameter space, offering stronger guarantees for inference in scientific settings where one has to ensure the reliability of conclusions regardless of the specific source that generated an observation.

\section{Supplementary material} \label{sec:supplementary_material}

\noindent We refer the reader to the following supplementary online material at \url{https://lee-group-cmu.github.io/tsi/} for additional results, and for details on the synthetic examples and case studies:\\
\begin{enumerate}
    \item Supplement on 1D synthetic example
    \item Supplement on 2D synthetic examples
    \item Supplement on Waldo versus FreB
    \item Supplement on case study I
    \item Supplement on case study II
    \item Supplement on case study III
    \item Supplementary figures
\end{enumerate}




\section{Acknowledgments}

\noindent The authors would like to thank the STAtistical Methods for the Physical Sciences (STAMPS) Research Center at Carnegie Mellon University for support. ABL is grateful to Mikael Kuusela, Jing Lei and Larry Wasserman for valuable discussions. This material is based upon work supported by NSF DMS-2053804, and the National Science Foundation Graduate Research Fellowship Program under Grant No DGE2140739. Any opinions, findings, and conclusions or recommendations expressed in this material are those of the authors and do not necessarily reflect the views of the National Science Foundation. RI is grateful for the financial support of FAPESP (grant 2023/07068-1) and CNPq (grants 305065/2023-8 and 403458/2025-0).

\section{Author contributions}

\noindent LM and ABL conceived the project and designed its components. LM, RI, and ABL designed the FreB protocol. LM and JC wrote the original code and carried out the synthetic examples. AS carried out case study I under the guidance of TD, MD and ABL. JC carried out case study II under the guidance of JS and ABL. JDI and ACHRJ carried out case study III under the guidance of JS and ABL. ABL together with JC, LM, JDI, JS, and RI took the lead in writing the manuscript. All authors reviewed and approved the manuscript.

\clearpage

\printbibliography[category=main,title={References}]

\clearpage

\appendix

\section{Theory and algorithms}\label{sec:theory_and_algos}

\subsection{Notation and formal problem set-up}
 Our assumption (well borne by the fundamental science use cases that we target) is that calibration data encode {\em the same} physical process 
as target data. Hence, we also assume that the likelihood function $p(\x|\theta)$ with $\theta \in \Theta$ and $\x \in \mathcal{X}$, which describes the data-generating process, is the same for calibration and target data. We refer to the label distribution $\pi(\theta)$ on the train data as our {\em prior distribution}.  The {\em reference distribution} $r(\theta)$ on the calibration set is a distribution that dominates the prior distribution, $r \gg \pi$. The prior $\pi(\theta)$ can be different from the label distribution $p_{\rm target}(\theta)$ of the target data, as well as different from the reference distribution $r(\theta)$ of the calibration set. Morever, the train data distribution $\widehat{p} (\x|\theta)$ can be different from $p(\x|\theta)$. See Section 3.1 for our experimental set-up.\\
\\
Now let 
 $p(\x)  := \int \widehat{p}(\x|\theta) \pi(\theta) d\theta$ be the marginal probability density function of $\X$ on train data. 
Our {\em posterior distribution} on the train data is then defined as  
  $\widehat \pi(\theta|\x):= { \widehat p(\x|\theta) \pi(\theta) }/{p(\x)}$; that is, the posterior is the conditional density of $\theta$ given $\x$ on train data.\\


\begin{Definition}[Confidence procedure] \label{def:conf_procedure} 
Let $\mathcal{A}$ denote the space of all measurable sets,
$\mathcal{A} \subseteq  \mathcal{X}  \times \Theta$. A {\em confidence procedure} is a set $\C$ in the space $\mathcal{A}$ defined as
$$ \{(\x, \theta) : (\x, \theta) \in \C \}.$$   For fixed $\x$, we define the confidence set or $\theta$-section as 
$$C(\x)=\{\theta: (\x, \theta) \in \C\}.$$ For fixed $\theta$, we define the acceptance region or $\x$-section as  
$$C_\theta=\{\x: (\x, \theta) \in \C\}.$$ 
A $(1-\alpha)$ confidence procedure is {\em valid} with respect to a distribution $p(\x|\theta)$ if, for every 
$\theta \in \Theta$ and every miscoverage level $0 \leq \alpha \leq 1$,
\begin{equation} 
\P_{\X|\theta}\left(\theta \in C(\X)\right) \geq 1-\alpha ,
 \label{eq:validity}
\end{equation}
 where $\P_{\X|\theta}$ is the conditional distribution of $\X$ given $\theta$ on the target data, $p(\x|\theta)$.
\end{Definition}



\subsection{From posteriors to confidence procedures
}
~\label{sec:confidence_procedures}

\noindent Let $ \widehat{\pi}(\theta|\X)$ be a posterior approximation based on the train data 
$$ \mathcal{T}_{\rm train} = \{ (\theta_1, \X_1) \, \ldots  (\theta_B, \X_B)\} \sim \pi(\theta) \widehat p(\x|\theta).$$
Once we have $ \widehat{\pi}(\theta|\X)$, it is straightforward to construct Bayesian credible regions for fixed $\x$ by computing highest posterior density (HPD) level sets  
\begin{align}
    \label{eq:hpd_def}
H_c(\x):=\left\{\theta:  \widehat{\pi}(\theta|\x)> c \right\},
\end{align}
where $ \int_{H_c(\x)} \widehat{\pi}(\theta|\x) d\theta = 1-\alpha$. These HPD sets however do not result in a valid confidence procedure (according to Definition~\ref{def:conf_procedure})  for train {\em or} target data. Moreover, even if the train and target distributions are exactly the same (with the same prior $\pi(\theta)$ and the same likelihood $p(\x|\theta)$), the HPD sets will only guarantee average or marginal validity. By construction,
\begin{align*} 
 \int_\Theta \P_{\X|\theta}\left(\theta \in H(\X)\right) \pi(\theta) d\theta 
 &= \int_\Theta \left(\int_{H_\theta} p(\x|\theta) d\x \right)\pi(\theta) d\theta \\
 &= \int_{\mathcal{X}} \left(\int_{H_c(\x)} \pi(\theta|\x) d\theta \right) p(\x) d\x  \\
 &\approx \int_{\mathcal{X}} \left(\int_{H_c(\x)} \widehat{\pi}(\theta|\x) d\theta \right) p(\x) d\x = 1-\alpha,
 \end{align*}
 where $H_\theta$ is the $\x$-section of a HPD confidence procedure with $1-\alpha$ credible sets $H_c(\x)$ at every $\x \in \mathcal{X}$.
\\ 
\\
\noindent In this paper, we propose a new approach that constructs confidence procedures that mirror the style of HPD level sets in Bayesian inference, while providing frequentist coverage properties for every $\theta \in \Theta$, regardless of $\pi(\theta)$. We apply a monotonic transformation $g_\theta$ to the posterior, so that the level sets $B_\alpha(\x) = \left\{\theta:  h(\x;\theta) > \alpha \right\}$,
 where $h(\x;\theta):=g_\theta(\widehat{\pi}(\theta|\x))$ control the type I error at level $\alpha$ for any $\theta \in \Theta$ and $0 < \alpha < 1$. In Appendix~\ref{sec:rej_prob}, we outline the construction of one such procedure that estimates  $h(\x;\theta)$ from the calibration set 
 $$ \mathcal{T}_{\rm cal} = \{ (\theta_1', \X_1') \, \ldots  (\theta_{B'}', \X_{B'}')\} \sim r(\theta) p(\x|\theta),$$
where we assume that $r \gg \pi$. \\

\noindent In Appendix~\ref{sec:p-values}, we show how confidence procedures can be constructed for all levels of miscoverage $\alpha$ simultaneously from an estimate of $g_\theta$.  
   Our procedure can be seen as a generalization of {\em confidence distributions} \citepsupp{schweder_confidence_2002,xie_confidence_2013,nadarajah_confidence_2015,cui_confidence_2023,thornton_bridging_2024} from one-dimensional to multidimensional parameter spaces $\Theta$. However, for many practical applications, researchers are only interested in constructing valid and precise confidence procedures for a {\em fixed prespecified} miscoverage level $\alpha$. In the latter case, one can reduce the complexity of the numerical estimation problem via an $\alpha$-level quantile regression of the test statistic on $\theta$. We outline the details of the latter approach in Appendix~\ref{sec:critical_values}.\\
\\

\subsection{Rejection probability across the entire parameter space}\label{sec:rej_prob} 

At the heart of our construction is the relationship between frequentist confidence sets $C(\X)$ and acceptance regions $C_{\theta_0}$ for tests of $H_{0, \theta_0}: \theta = \theta_0$ at all $\theta_0 \in \Theta$.  Below we define the rejection probability function $W$ for an arbitrary test statistic $\lambda$ that rejects $H_{0, \theta_0}$ for small values of the test statistic $\lambda$.

\begin{Definition}[Rejection probability] 
Let $\lambda$ be any test statistic; such as the estimated posterior, $\lambda(\X;\theta_0)=\widehat{\pi}(\theta_0|\X)$. The rejection probability of the test $H_{0, \theta_0}$ is defined as 
 \begin{align}
\label{eq:power_func}
W_\lambda(t; \theta, \theta_0) := \pr_{\X | \theta} \left(  \lambda(\X;\theta_0) \leq t \right),
\end{align}
where $\theta, \theta_0 \in \Theta$ and   $t \in \mathbb{R}$, and $\P_{\X|\theta}$ is the conditional distribution of $\X$ given $\theta$ on the target data, $p(\x|\theta)$.
\end{Definition}




We can learn the rejection probability function using a monotone regression that enforces the rejection probability to be a nondecreasing function of $t$. The computation is straightforward when $\theta=\theta_0$.  In this work, we propose a fast procedure for estimating the cumulative distribution function 
\begin{equation}\label{eq:lambda_CDF}
    F_\lambda(t; \theta_0):= W_\lambda(t; \theta_0, \theta_0) = \pr_{\X | \theta_0} \left(  \lambda(\X;\theta_0) \leq t \right)
\end{equation}
of the test statistic $\lambda$ as a function of the cut-off $t$ and the parameter value $\theta_0 \in \Theta$.
For each point $i$ ($i=1,\ldots, B'$) in the calibration set $ \mathcal{T}_{\rm cal} = \{ (\theta_1', \X_1') \, \ldots  (\theta_{B'}', \X_{B'}')\} \sim r(\theta) p(\x|\theta)$,
 we draw a sample of cutoffs $K$ according to the empirical distribution of the test statistic $\lambda$. Then, we regress the indicator variable 
\begin{align}
Y_{i,j}:= \I \left( \lambda(\X_i';\theta_i') \leq t_j \right )
\end{align}
on $\theta_i'$ and $t_{i,j}$ ($=t_j$) using  the ``augmented'' calibration sample  $\widetilde{\mathcal{T}}_{\rm cal}=\{(\theta_i',t_{i,j},Y_{i,j})\}_{i,j}$, for $i=1,\ldots,B'$ and $j=1,\ldots,K$, where $K$ is our augmentation factor. See Algorithm \ref{algo:rejection_prob0} for more details.

\begin{algorithm} 
    \caption{
    \texttt{Learning the rejection probability function}}\label{algo:rejection_prob0}
    \begin{flushleft}
    \textbf{Input:} {\small test statistic $\lambda$; calibration data  $\mathcal{T}_{\rm cal}=\{(\theta'_1, \X'_1),\ldots,(\theta'_{B'}, \X'_{B'})\}$; oversampling factor K; 
    evaluation points $\mathcal{V} \subset \Theta$}\\
    \textbf{Output:} {\small Estimate of rejection probability $F_\lambda(t;\theta)$ when $\theta=\theta_0$, for all $t \in G$ and $\theta \in \mathcal{V}$ }
    \end{flushleft}		
    \begin{algorithmic}[1]
        \State \codecomment{Learn rejection probability from augmented calibration data  $\widetilde{\mathcal{T}}'$}
        \State Set $\widetilde{\mathcal{T}}_{\rm cal} \gets \emptyset$
        \State Let $G_0 \gets \{\lambda(\X'_1;\theta_1'), \ldots,\lambda(\X'_{B'};\theta'_{B'})\}$
        \For{$i$ in $\{1,...,B'\}$}	
                \State Let $G \gets \text{sample of size K from } G_0$\text{ with replacement}
            \For{$j$ in $\{1,...,K\}$}  
                \State Let $t_j \gets G[j]$
                \State Compute $Y_{i,j} \gets \I \left(  \lambda(\X'_i;\theta_i') \leq t_{j} \right)$
                \State Let $ \widetilde{\mathcal{T}}_{\rm cal} \gets  \widetilde{\mathcal{T}}_{\rm cal}  \cup \{\left(\theta'_i,t_{j}, Y_{i,j} \right)\}$ 
           \EndFor
        \EndFor  
        \State Estimate  $ F_\lambda(t;\theta) := \pr_{\X| \theta} \left(   \lambda(\X;\theta) \leq t \right)$ from $\widetilde{\mathcal{T}}_{\rm cal}$ via a regression of $Y$ on $\theta$ and $t$, which is monotonic in $t$. 
	\State \textbf{return}  estimated rejection probabilities $\widehat{F}_\lambda(t;\theta)$, for $t \in G$, $\theta \in \mathcal{V}$ 	
    \end{algorithmic}
\end{algorithm}

\subsection{Amortized p-values for constructing confidence procedures} 
\label{sec:p-values}

For any test statistic $\lambda$ and null hypothesis $H_{0,\theta_0}:\theta=\theta_0$, we can define a new test statistic $h$ via a monotonic transformation,
\begin{align}
    h(\X; \theta_0) &:= F_\lambda( \lambda(\X; \theta_0) ; \theta_0), \nonumber \\
    & = \P_{\x | \theta_0}\left( \lambda(\x; \theta_0) <   \lambda (\X ; \theta_0) \right) ,\label{eq:p-value}
 \end{align} 
and then a corresponding family of confidence sets of $\theta$ by taking level sets,
$$	B_\alpha(\X) = \left\{ \theta_0 \in \Theta \mid   h(\X; \theta_0)> \alpha \right\} , $$
 where $0 \leq \alpha \leq 1$. The following theorem shows that $F_\lambda$ (Equation~\eqref{eq:lambda_CDF}) is the only monotonic transformation that controls type I errors; that is, makes $h(\X;\theta_0)$ a {\em valid} p-value with level sets $B_\alpha(\X)$ that are confidence sets with frequentist level-$\alpha$ coverage.

\begin{thm}\label{th:pvalue_uniqueness}
Let $\lambda(\x;\theta)$ be any test statistic. 
For every fixed $\theta \in \Theta$, let $g_\theta: \mathbb{R} \longrightarrow \mathbb{R}$ be a monotonic transformation of $\lambda(\x;\theta)$.
Then 
$$\P_{\X|\theta}\left(g_\theta(\lambda(\X;\theta))>\alpha \right)=1-\alpha \text{ for every $\alpha \in (0,1)$ and $\theta \in \Theta$}$$
if, and only if,
$g_\theta(\lambda(\x;\theta))= F_\lambda( \lambda(\x; \theta) ; \theta)$.
\end{thm}

\begin{proof}\ \ 

$\Rightarrow$  direction: Fix $\theta$ and let $g_\theta$ be any monotonic transformation for $\lambda$ as stated in the theorem. Then
\begin{align*}
    \P_{\X|\theta}&\left(g_\theta(\lambda(\X;\theta))>\alpha \right)=1-\alpha, \ \forall \alpha \in (0,1) \\
    \iff  &\P_{\X|\theta}\left(\lambda(\X;\theta)>g^{-1}_\theta(\alpha) \right)=1-\alpha,  \ \forall \alpha \in (0,1)  \\
     \iff  &\P_{\X|\theta}\left(\lambda(\X;\theta)\leq g^{-1}_\theta(\alpha) \right)=\alpha,    \ \forall \alpha \in (0,1)  \\
         \iff  &F_{\lambda}(g^{-1}_\theta(\alpha) ; \theta)   =\alpha,  \ \forall \alpha \in (0,1) \\
         \iff &g^{-1}_\theta(\alpha)=F^{-1}_{\lambda}(\alpha;\theta),  \ \forall \alpha \in (0,1)  \\
         \iff &g_\theta(\lambda(\x;\theta))= F_\lambda( \lambda(\x; \theta) ; \theta),  \ \forall \x \in \mathcal{X} .
\end{align*}
$\Leftarrow$ direction: Let $g_\theta(\lambda(\x;\theta))= F_\lambda( \lambda(\x; \theta) ; \theta)$. Notice that
\begin{align*}
    \P_{\X|\theta}\left(g_\theta(\lambda(\X;\theta))>\alpha \right)&= \P_{\X|\theta}\left(F_\lambda( \lambda(\X; \theta) ; \theta)>\alpha \right)\\
    &= \P_{\X|\theta}\left( \lambda(\X; \theta)  > F^{-1}_\lambda( \alpha;\theta )\right)  \\
    &= 1-\P_{\X|\theta}\left( \lambda(\X; \theta)  \leq F^{-1}_\lambda( \alpha;\theta )\right)  \\   &= 1-F_{\lambda}(F^{-1}_\lambda( \alpha;\theta); \theta)   \\
    &=
    1-\alpha.
\end{align*}
\end{proof}

\paragraph{Confidence procedures at all levels $\alpha$ simultaneously.}  To summarize: Algorithm~\ref{algo:rejection_prob0} offers a means to computing p-values    $\widehat{h}(\x; \theta_0) := \widehat F_\lambda( \lambda(\x; \theta_0) ; \theta_0)$ and the entire family of confidence sets $\widehat{B}_\alpha(\x) := \left\{ \theta \in \Theta \mid   \widehat{h}(\x; \theta_0)> \alpha \right\} $, which is fully amortized with respect to observed data $\x \in \mathcal{X}$, the parameter $\theta_0\in \Theta$, and the miscoverage level $0 \leq \alpha \leq 1$. That is,  once we have the test statistic $\lambda(\x;\theta_0)$ and the rejection probability  $\widehat F( t ; \theta_0)$ as a function of all $t \in \mathbb{R}$ and  $\theta_0 \in \Theta$ (via Algorithm~\ref{algo:rejection_prob0}), we can perform inference for new data without retraining for all miscoverage levels $\alpha$ simultaneously.\\ 


\subsection{Alternative construction of confidence procedures at a fixed prespecified level} \label{sec:critical_values}
For many practical applications, researchers are only interested in constructing valid and precise confidence procedures with 
\begin{align}\label{eq:confidence_set}
    \widehat{B}_\alpha(\x) &:= \left\{ \theta \in \Theta \mid  \widehat F_\lambda\left( \lambda(\x; \theta);\theta\right) > \alpha \right\} \nonumber\\
    &= \left\{ \theta \in \Theta \mid   
     \lambda(\x; \theta) > \widehat F^{-1}_\lambda (\alpha;\theta)
    \right\}
\end{align}
for some pre-specified miscoverage level $\alpha \in (0,1)$. In such cases, we only need to estimate the critical values 
$t_{\theta_0} := F^{-1}_\lambda (\alpha;\theta_0)$ for a fixed level-$\alpha$ test of $H_0: \theta=\theta_0, \ \forall \theta_0 \in \Theta$. Algorithm \ref{alg:estimate_cutoffs} outlines an amortized approach that estimates the critical values across the parameter space; this algorithm was first proposed by \citesupp{dalmasso_confidence_2020} for approximate likelihood approaches.

\begin{algorithm} 
    \small
    \caption{Estimate critical values $t_{\theta_0}$  for a level-$\alpha$ test of  $H_{0, \theta_0}: \theta = \theta_0$ vs. $H_{1, \theta_0}: \theta \neq \theta_0$ for all $\theta_0 \in \Theta$ simultaneously}
    \begin{flushleft}
    {\bf Input:} 
     test statistic $\lambda$; calibration data $\mathcal{T}_{\rm cal}=\{(\theta'_1, \X'_1),\ldots,(\theta'_{B'}, \X'_{B'})\}$; quantile regression estimator; level $\alpha \in (0,1)$\\
    {\bf Output:} estimated critical values $\widehat{t}_{\theta_0}$ for all $\theta_0 \in \Theta$
    \end{flushleft}
    \label{alg:estimate_cutoffs} 
    \begin{algorithmic}[1]
    \State Set $\widetilde{\T}_{\rm cal} \gets \emptyset$ 
    \For{i in $\{1,\ldots, B'\}$}
    \State Compute test statistic $\lambda'_i \gets \lambda(\X'_{i};\theta'_i) $
    \State $\widetilde{\T}_{\rm cal} \gets \widetilde{\T}_{\rm cal} \cup  \{(\theta'_i,\lambda'_{i})\}$  
    \EndFor
    \State  Use $\widetilde{\T}_{\rm cal}$ to learn 
     the conditional quantile function $\widehat{t}_\theta := \widehat{F}_{\lambda|\theta}^{-1}(\alpha | \theta)$
    via quantile regression of  $\lambda$ on $\theta$ \\
	\textbf{return}
	 $\widehat{t}_{\theta_0}$
	\end{algorithmic}
\end{algorithm}



\subsection{Validity of Frequentist-Bayes procedure}
\label{sec:validity}

\subsubsection{P-value estimation}


%

The method of estimating the p-value  described in Appendix \ref{sec:p-values} is consistent. Below we adapt the general LF2I results in \citecsupp{Sec  4.2}{dalmasso_likelihood-free_2024} which hold in general, even for fully amortized procedures (Algorithm~\ref{algo:rejection_prob0}). The proofs are equivalent.

\begin{Assumption}[Uniform consistency]
\label{assump:uniform_consistency}
The regression estimator used in Algorithm \ref{algo:rejection_prob0}  is such that
$$\sup_{\theta,t} |\hat \E_{B'}[Y|\theta,t]- \E[Y|\theta,t]|  \xrightarrow[B' \longrightarrow\infty]{\enskip \text{a.s.} \enskip}  0.$$
\end{Assumption}

\noindent
If $\Theta$ is continuous and the Lebesgue measure dominates $r$, then the estimators described, e.g., in \citepsupp{bierens_uniform_1983,hardle_uniform_1984,liero_strong_1989,girard_uniform_2014} satisfy this assumption.


\begin{thm}
\label{thm:pval_right_coverage}
Fix $\theta_0 \in \Theta$. Under Assumption \ref{assump:uniform_consistency}  and if $h(\X;\theta_0)$ is  an absolutely continuous random variable then, for every $\theta \in \Theta$,
$$\widehat h(\X;\theta_0) \xrightarrow[B' \longrightarrow\infty]{\enskip \text{a.s.} \enskip} h(\X;\theta_0)$$
and 
$$\P_{\X,\mathcal{T}'|\theta}\left(\hat h\left(\X;\theta_0\right)\leq \alpha\right) \xrightarrow{B' \longrightarrow\infty} \P_{\X|\theta}(h (\X;\theta_0)\leq \alpha).$$
In particular,
$$\P_{\X,\mathcal{T}'|\theta_0}\left(\hat h\left(\X;\theta_0\right)\leq \alpha\right) \xrightarrow{B' \longrightarrow\infty}  \alpha.$$
\end{thm}

\begin{Assumption}[Convergence rate of the regression estimator]
\label{assump:conv_reg_pval}
The regression estimator is such that
$$\sup_{\theta,t} |\hat \E[Z|\theta,t]- \E[Z|\theta,t]|=O_P\left(\left(\frac{1}{B'}\right)^{r}\right).$$
for some $r>0$.
\end{Assumption}

\noindent
Examples of regression estimators that satisfy Assumption \ref{assump:conv_reg_pval} 
when $\Theta$ is continuous and the Lebesgue measure dominates $r$ can be found in \citesupp{stone_optimal_1982,hardle_uniform_1984,donoho_asymptotic_1994,yang_frequentist_2017}.

\begin{thm}
\label{thm:pval_rate}
Under Assumption \ref{assump:conv_reg_pval}, 
$$|\hat h(\X;\theta_0)- h(\X;\theta_0)|=O_P\left(\left(\frac{1}{B'}\right)^{r}\right).$$
\end{thm}

\begin{proof}[Proof of Theorem \ref{thm:pval_rate}]
The result follows directly from Assumption \ref{assump:conv_reg_pval}
and the fact that
$\widehat{h}(\x; \theta_0) := \widehat F_\lambda( \lambda(\x; \theta_0) ; \theta_0)=\widehat \E \left[Z|\theta_0,\lambda(\x; \theta_0) \right].$
\end{proof}

\subsubsection{Critical value estimation}\label{eq:validity_criticalvalues}


Our procedure for choosing critical values leads to
valid hypothesis tests (that is, tests that control the type I error probability), as long as the number of simulations $B'$ in Algorithm \ref{alg:estimate_cutoffs} is sufficiently large. See \citecsupp{Sec 4.1}{dalmasso_likelihood-free_2024} and Appendix \ref{sec:diagnostics} for details.

 \begin{Assumption}[Uniform consistency]
\label{assum:quantile_consistent_simple_null} 
Let $\hat F_{B'} (\lambda;\theta)$
be the estimated cumulative distribution function of the test statistics $\lambda$ indexed by $\theta$, implied by Algorithm \ref{alg:estimate_cutoffs}.
Assume that the quantile regression estimator is such that
$$\sup_{\lambda \in \mathbb{R}}|\hat F_{B'}(\lambda;\theta_0)-  F(\lambda;\theta_0)|\xrightarrow[B' \longrightarrow\infty]{\enskip P \enskip} 0.$$
\end{Assumption}

\noindent
Assumption~\ref{assum:quantile_consistent_simple_null} holds, for instance, for quantile regression forests \citepsupp{meinshausen_quantile_2006}.

\noindent
Next, we show that Algorithm \ref{alg:estimate_cutoffs} yields a valid hypothesis test as $B' \rightarrow \infty$.

\begin{thm}
 \label{thm:valid_tests}
Let 
$C_{B'}=\hat F_{B'} (\alpha;\theta_0)$. If the quantile
estimator satisfies Assumption~\ref{assum:quantile_consistent_simple_null},
then, for every $\theta_0 \in \Theta$,
$$ \P_{\X|\theta_0,C_{B'}}(\lambda(\X;\theta_0) \leq C_{B'})  \xrightarrow[B' \longrightarrow\infty]{\enskip a.s. \enskip}   \alpha,$$
where $\P_{\X|\theta_0,C_{B'}}$ denotes the probability integrated over $\X\sim p(\x|\theta_0)$ and conditional on the random variable $C_{B'}$.
\end{thm}

\noindent
If the convergence rate of the quantile regression estimator is known (Assumption \ref{assum:quantile_regression_rate}),  Theorem \ref{thm:valid_tests_rate} provides a finite-$B'$ guarantee on how far the type I error of the test will be from the nominal level.

\begin{Assumption}[Convergence rate of the quantile regression estimator]
\label{assum:quantile_regression_rate}
Using the notation of Assumption \ref{assum:quantile_consistent_simple_null}, assume that the quantile regression estimator is such that
$$\sup_{\lambda \in \mathbb{R}}|\hat F_{B'}(\lambda;\theta_0)-  F(\lambda;\theta_0)|=O_P\left(\left(\frac{1}{B'}\right)^{r}\right)$$
for some $r>0$.
\end{Assumption}

\begin{thm}
 \label{thm:valid_tests_rate}
With the notation and assumptions of Theorem \ref{thm:valid_tests}, and if  Assumption~\ref{assum:quantile_regression_rate} also holds,
then,
$$ |\P_{\X|\theta_0,C_{B'}}(\lambda(\X;\theta_0) \leq C_{B'}) - \alpha| =O_P\left(\left(\frac{1}{B'}\right)^{r}\right).$$
\end{thm}



\subsection{Power of Frequentist-Bayes procedure}
\label{sec:power}

Consider a confidence procedure $\B \in \Theta \times \mathcal{X}$ with $\theta$-sections at fixed   $\x \in \mathcal{X}$ and  $\alpha \in (0,1)$ defined by
\begin{align}
\label{eq:BF_set}
B_\alpha(\x) = \left\{ \theta \in \Theta \mid   h(\x; \theta)> \alpha \right\},
\end{align}
 where  $h(\x; \theta)$ is the p-value (Equation~\ref{eq:p-value}) for the test statistic $\lambda(\x;\theta)=\pi(\theta|\x)$. In Appendix~\ref{sec:validity}, we show that $\B$ is a valid confidence procedure on both calibration and target data, regardless of the choice of prior $\pi(\theta)$, satisfying $\P_{\X|\theta}(\theta \in B_\alpha(\X)) = 1-\alpha, \ \forall \theta \in \Theta$. In this section, we show that 
  if $\widehat p(\x|\theta)=p(\x|\theta)$ (that is, if the training set has the same likelihood function as the target set, then $B_\alpha(\x)$ has a small expected size $$ \E \left(|B_\alpha(\X)|\right) := \int_{\mathcal{X}} \left(\int_{B_\alpha(\x)}  d\theta\right) p(\x) d\x$$  with respect to the marginal distribution $p(\x) = \int  p(\x|\theta) \pi(\theta) d\theta$.  Different versions of this theorem have appeared in e.g.
\citesupp{pratt_length_1961,yu_adaptive_2018,hoff_bayes-optimal_2023} for continuous $\Theta$, as well as \citesupp{sadinle_least_2019} when $\Theta$ is finite.
\\
\\
\noindent It follows directly that if the training set has the same likelihood $p(\x|\theta)$ as the target data, {\em and} the design prior $\pi$ is ``well-specified'' and places a high mass around the true parameter value $\theta$ for the target data according to $\pi(\theta)=p_{\rm target}(\theta)$, then the frequentist Bayes sets $B_\alpha(\x)$ will not only achieve nominal coverage across the parameter space $\Theta$; they will also on average be smaller than any other valid confidence sets with respect to the marginal distribution
 $p_{\rm target}(\x)$ of the target data. However, if the prior is different from the (unknown) label distribution or ``true prior'' $p_{\rm target}(\theta)$ of the target data, then frequentist Bayes sets will not have optimal average constraining power with respect to $p_{\rm target}(\x)$. 

\begin{Lemma}[Neyman-Pearson lemma]
\label{lemma::NP} Let $\mu(\z)$ and $\nu(\z)$ be nonnegative functions in $L_1$.  Fix $\alpha \in (0,1)$, and assume that there exists $t$ such that the set $A^* =\{\z: \mu(\z)/\nu(\z)\geq t\}$ satisfies $\mu(A^*)=1-\alpha$.   Then $A^*$ is the solution to the following optimization problem:
$$\min_A \int_A \nu(\z)d\z \ \text{ subject to } \int_A \mu(\z) d\z \geq 1-\alpha.$$
\end{Lemma}

\begin{thm}
\label{thm:best_region}
Let $\mathcal{A}$ denote the space of all measurable sets
$A \subseteq \Theta  \times \mathcal{X}$, and let $A(\x)=\{\theta: (\theta,\x) \in A\}$ be the $\theta$-section of $A$, and let $|A(\X)|=\int_{A(\X)}  d\theta$ be the size of $A(\X)$.  Let $A^*$  be the solution to the following minimization problem:
 $$\min_{A \in \mathcal{A}} \E \left[|A(\X)|\right] \ \text{ subject to } \P_{\X|\theta}(\theta \in A(\X))\geq 1-\alpha, \ \forall \theta \in \Theta,$$
where the expectation is taken with respect to the marginal distribution $p(\x)=\int p(\x|\theta)\pi(\theta)d\theta$.
 Then, if $\widehat p(\x|\theta)=p(\x|\theta)$, we have $A^*(\x)=B_\alpha(\x)$ (Equation \ref{eq:BF_set}).
\end{thm}

\begin{proof}\  
Let $A_\theta=\{\x: (\theta,\x) \in A\}$ be the $\x$-section of $A$.
Notice that the optimization problem is equivalent to  
$$\min_{A \in \mathcal{A}} \int \left[\int_{A(\x)} 1 d\theta  \right] p(\x)d\x \ \text{ subject to } \int_{A_\theta} p(\x|\theta)d\x \geq 1-\alpha \ \forall \theta \in \Theta,$$
which is further equivalent to
$$\min_{A \in \mathcal{A}} \int \left[\int_{A_\theta}  p(\x)d\x \right]  d\theta \  \text{ subject to } \int_{A_\theta} p(\x|\theta)d\x \geq 1-\alpha \ \forall \theta \in \Theta,$$
which is equivalent to a point-wise optimization problem for any given $\theta$:
$$\min_{A_\theta}  \int_{A_\theta}  p(\x)d\x    \ \text{ subject to } \int_{A_\theta} p(\x|\theta)d\x \geq 1-\alpha.$$
Lemma \ref{lemma::NP} implies that the optimal solution is
$$A^*_\theta=\{x:p(\x|\theta)/p(\x) \geq t_\theta\},$$
where $t_\theta$ satisfies $\P_{\X|\theta}(\theta \in A^*(\X))=1-\alpha$. The optimal set is then (using the fact that if $\widehat p(\x|\theta)=p(\x|\theta)$, then  $p(\x|\theta)/p(\x) =\pi(\theta|\x)/\pi(\theta)$)
$$A^*=\{(\theta,\x):\pi(\theta|\x)/\pi(\theta) \geq t_\theta\},$$
or, equivalently,
$$A^*=\{(\theta,\x):\pi(\theta|\x)\geq t'_\theta\},$$
where $t'_\theta=t_\theta \pi(\theta)$.
\end{proof}

\subsection{Local diagnostics to check coverage across the parameter space}
\label{sec:diagnostics}


\begin{algorithm}[!h]
    \small \caption{Estimate empirical coverage $ \pr_{\X|\theta}(\theta \in \widehat B_\alpha(\X))$, for all $\theta \in \Theta$.}
    \label{alg:estimate_coverage}
    \begin{flushleft}
        {\bf Input:} simulator $F_{\theta}$; number of simulations $B''$; $\pi_{\Theta}$ (fixed proposal distribution over parameter space); test statistic $\lambda$; level $\alpha$; critical values $\widehat{C}_\theta$; probabilistic classifier\\
        {\bf Output:} estimated coverage $ \widehat \pr_{\X|\theta}(\theta \in \widehat B_\alpha(\X))$ for all $\theta  \in \Theta$
    \end{flushleft}
    \begin{algorithmic}[1]
    \State Set $\T_{\rm diagn} \gets \emptyset$ 
    \For{i in $\{1,\ldots,B^{''}\}$}
    \State  Draw parameter $\theta_i \sim  \ r(\theta)$
    \State Draw sample 
    $\X_i \stackrel{iid}{\sim} p(\x|\theta_i)$
    \State Compute  test statistic $\lambda_i \gets \lambda(\X_i;\theta_i)$
    \State Compute indicator  variable $W_i \gets \I\left(\lambda_i \geq \widehat{C}_{\theta_i}\right)$  \State $\T_{\rm diagn} \gets \T_{\rm diagn} \cup \{(\theta_i, W_i)\}$
    \EndFor
    \State  Use $\T_{\rm diagn}$ to learn $ \widehat \pr_{\X|\theta}(\theta \in \widehat B_\alpha(\X))$ across $\Theta$ by regressing  $W$ on $\theta$ with a probabilistic classifier
    \State \textbf{return}  
    $ \widehat \pr_{\X|\theta}(\theta \in \widehat B_\alpha(\X))$
    \end{algorithmic}
\end{algorithm}

\clearpage

\section{2D example with a well-specified forward model}\label{sec:well_specified_2D}


\noindent In this example, we follow up on the result presented in Section~\ref{sec:2D_example} seen in Fig.~\ref{fig:2D_example}. As was the case there, the true likelihood is given by a mixture of two normal distributions,
\begin{equation*}
    p(X\vert\theta) = \frac{1}{2}\mathcal{N}(\theta, I) + \frac{1}{2}\mathcal{N}(\theta, \sigma^2 I),
\end{equation*}
where $\sigma=0.1$, and the common mean $\theta$ is the parameter of interest. However, this time we assume that the posterior was learned using train data from the correctly specified joint distribution on $\theta$ and $X$ alike,
 $$ \mathcal{T}_{\rm train} = \{ (\theta_1, X_1) \, \ldots  (\theta_B, X_B)\} \sim \pi(\theta) p(X | \theta),$$
with the same localized prior $\pi(\theta) = \mathcal{N}(0, 2I)$ and a well-specified forward model (i.e.~$\delta=0$, by Equation~\ref{eq:misspecified_forward_model}). Figure~\ref{fig:2D_example_well_specified} Panel a shows that HPD sets from a flow matching estimator trained with $B=50{,}000$ still fail to provide instance-wise coverage of the true parameter on average except near the mode of the prior, drastically noted at $\theta^*=(8.5, 8.5)$ in Panel a-\textit{left}. With the same calibration data distribution as in Section~\ref{sec:2D_example} ($B'=30,000$) and using the same monotone neural network architecture to learn the p-value function, the FreB procedure accomplishes the same outcome of providing valid confidence sets regardless of the value of the true parameter. Panel b shows the reshaped confidence sets resulting from FreB, and Panel b-\textit{right} indicates that the $95\%$ coverage probability is still maintained everywhere.\\

\begin{figure}[hb!]
    \centering
    \includegraphics[width=1.0\columnwidth,trim=0 0 1.4in 0,clip]{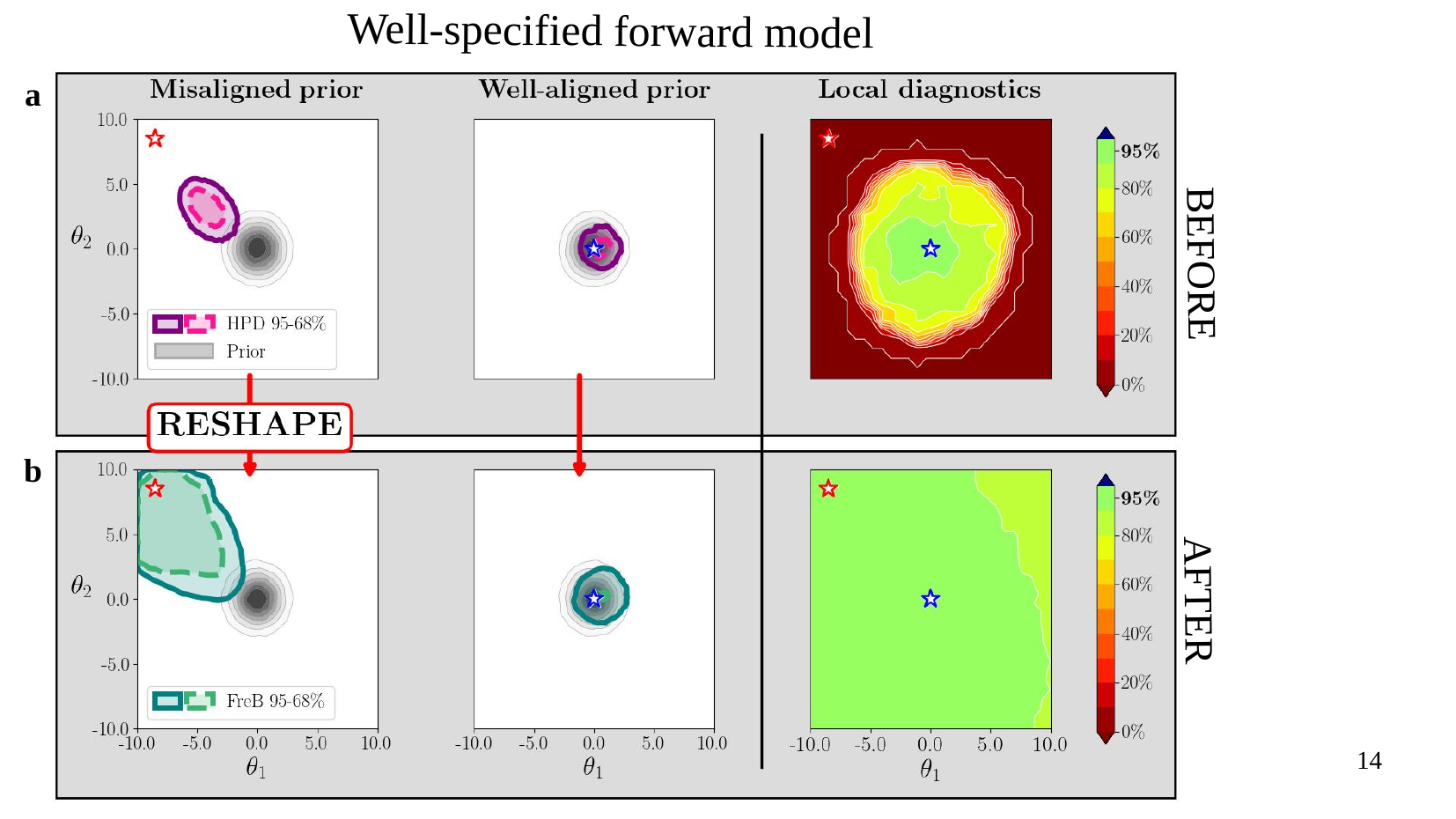}
    \caption{\small \textbf{2D synthetic Gaussian mixture model example with a well-specified forward model.} The task is to infer the common mean $\theta$ of a mixture of two Gaussians with different covariances using a flow matching generative model trained with a localized prior centered at the origin. \textbf{Panel a:} 95\% and 68\% HPD sets for two scenarios where the prior is misaligned ({\em left}) versus well-aligned {(\em center}) with the true $\theta^*$  (\textcolor{red}{red} star). {\em Right,} Local diagnostics of 95\% HPD sets shows that the actual coverage of these sets can be very far from the nominal 95\% level, when the truth is further away from the center where the train data are concentrated.
   \textbf{Panel b:} After reshaping and slicing the posteriors as in Figure~\ref{fig:1D_example}b, we obtain the corresponding FreB sets. 
   For all instances of $\theta$ and for all levels of $\alpha$, domain scientists are guaranteed to achieve the desired coverage level, here illustrated for the 95\% case in the {\em right} plot. That is, FreB sets are robust against misaligned training priors. The size of FreB sets is also smaller for well-aligned priors (compare {\em center bottom} plot with the {\em left bottom} plot).}
    \label{fig:2D_example_well_specified} 
\end{figure}

\clearpage

\printbibliography[category=supplement,title={Appendix references}]

\end{document}